%% file: PaperForReview.tex
\documentclass[10pt,twocolumn,letterpaper]{article}
\input{math_commands.tex}

\usepackage{cvpr}                   
\usepackage[accsupp]{axessibility}  

\usepackage{graphicx}
\usepackage{subcaption}
\usepackage{amsmath, bm}
\usepackage{amssymb}
\usepackage{booktabs}
\usepackage{float}
\usepackage{xcolor, colortbl}
\definecolor{Gray}{gray}{0.93}

\usepackage{algorithm}
\usepackage{algpseudocode}
\usepackage[hang]{footmisc}

\usepackage{amsthm}
\usepackage[english]{babel}

\newtheorem{prop}{Proposition}
\usepackage[misc]{ifsym}
\usepackage{bbding}

\newcommand{\Tabref}[1]{Table~\ref{#1}}
\renewcommand{\eqref}[1]{Eqn.~(\ref{#1})}

\usepackage {hyperref} 
\hypersetup {pagebackref,breaklinks,colorlinks}

\usepackage[capitalize]{cleveref}
\crefname{section}{Sec.}{Secs.}
\Crefname{section}{Section}{Sections}
\Crefname{table}{Table}{Tables}
\crefname{table}{Tab.}{Tabs.}

\title{Class-Balancing Diffusion Models}

\newcommand\blfootnote[1]{%
  \begingroup
  \renewcommand\thefootnote{}\footnote{#1}%
  \addtocounter{footnote}{-1}%
  \endgroup
}

\author{Yiming Qin$^1$
\and
Huangjie Zheng$^2$
\and
Jiangchao Yao$^{1,3}$\Envelope 
\and
Mingyuan Zhou$^2$
\and
Ya Zhang$^{1,3}$
}
\date{
    $^1$Cooperative Medianet Innovation Center, Shanghai Jiao-Tong University\\
    $^2$University of Texas, Austin\\
    $^3$Shanghai AI Laboratory
}

\begin{document}
\maketitle

\begin{abstract}

Diffusion-based models have shown the merits of generating high-quality visual data while preserving better diversity in recent studies. However, such observation is only justified with curated data distribution, where the data samples are nicely pre-processed to be uniformly distributed in terms of their labels. In practice, a long-tailed data distribution appears more common and how diffusion models perform on such class-imbalanced data remains unknown. In this work, we first investigate this problem and observe significant degradation in both diversity and fidelity when the diffusion model is trained on datasets with class-imbalanced distributions. Especially in tail classes, the generations largely lose diversity and we observe severe mode-collapse issues. To tackle this problem, we set from the hypothesis that the data distribution is not class-balanced, and propose Class-Balancing Diffusion Models (CBDM) that are trained with a distribution adjustment regularizer as a solution. Experiments show that images generated by CBDM exhibit higher diversity and quality in both quantitative and qualitative ways. Our method benchmarked the generation results on CIFAR100/CIFAR100LT dataset and shows outstanding performance on the downstream recognition task.

\end{abstract}

\blfootnote{The code is available at: \url{https://github.com/qym7/CBDM-pytorch}}

\input{section1}

\label{sec:intro}

\input{section2}

\label{sec:relwork}

\input{section3}

\label{sec:approach}

\input{section4}

\label{sec:experiments}
\section{Conclusion}

In this paper, we focus on the problem when the training data is imbalanced and the diffusion model drops in generation quality on  tail classes. We first establish a low baseline for this task by examining the most common paradigms in long-tail scenarios and in training generative models with limited data. Thereafter, we propose the Class-Balancing Diffusion Model (CBDM) through theoretical analysis. This approach can be implemented very cleanly in the training of any conditional diffusion model and thus has the potential to be widely used in other fields. Our experiments show that the CBDM approach significantly improves model generation diversity with high fidelity, on both class-balanced and class-imbalanced datasets. 

\section{Acknowledgement}
This work is supported by the National Key R\&D Program of China (No. 2022ZD0160703, No. 2022ZD0160702), National Natural Science Foundation of China (62271308), STCSM (No. 22511106101, No. 18DZ2270700, No. 22511105700, No. 21DZ1100100), 111 plan (No. BP0719010), and State Key Laboratory of UHD Video and Audio Production and Presentation.

\newpage

{\small
\bibliographystyle{ieee_fullname}
\bibliography{egbib}
}

\input{PaperSuppMaterials}

\end{document}

%% file: math_commands.tex
\usepackage{amsmath,amsfonts,bm}

\newcommand{\given}{\,|\,}

\def\Figref#1{Figure~\ref{#1}}

\def\eqref#1{equation~\ref{#1}}

\def\Algref#1{Algorithm~\ref{#1}}

\def\1{\bm{1}}

\def\rvepsilon{{\boldsymbol{\epsilon}}}

\def\rvx{{\boldsymbol{x}}}

\def\vzero{{\bm{0}}}

\def\mI{{\bm{I}}}

\DeclareMathAlphabet{\mathsfit}{\encodingdefault}{\sfdefault}{m}{sl}
\SetMathAlphabet{\mathsfit}{bold}{\encodingdefault}{\sfdefault}{bx}{n}

\def\gL{{\mathcal{L}}}

\def\gN{{\mathcal{N}}}

\def\gY{{\mathcal{Y}}}

\newcommand{\E}{\mathbb{E}}

\newcommand{\KL}{D_{\mathrm{KL}}}



%% file: section1.tex
\section{Introduction}

\begin{figure}
    \centering
     \includegraphics[width=0.98\linewidth]{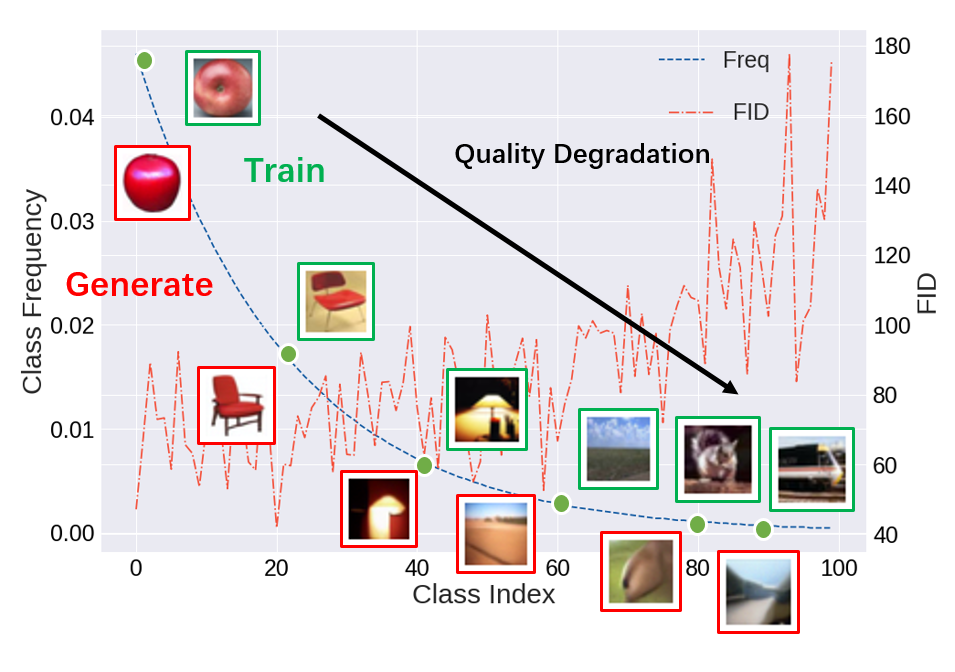}
    \caption{Generation degrades along with class frequency. Semantics of generated images become less recognizable when class frequency decreases, while the FID score increases significantly.}
    \label{fig:ltdataset}
\end{figure}

In recent years, log-likelihood-based diffusion models have evolved rapidly and established new benchmarks on a range of generation tasks~\cite{guided,bahdat2022lsgm}.
Based on them, researchers have been able to further control the model generation process and the generation quality. This improves the applications of generative models in numerous domains including text-image generation\cite{glide}, image editing\cite{saharia2022palette,meng2021sdedit}, speech synthesis\cite{diffspeech}, medical imaging\cite{song2022medical,luo2022medicalmri}, video generation\cite{videogeneration} and adversarial learning\cite{adversarialdiff,wu2022guidedpurification}, etc.

Although diffusion models are known for the power of high fidelity and diversity in generation, most of the existing diffusion models are trained with the hypothesis that the data are uniformly distributed \textit{w.r.t.} their labels. However, in the real world, the distribution is often very skewed. Especially for many domain-specific generation tasks such as medical images\cite{medicaldata}, fine-grained dataset for taxology\cite{inature} and data grabbed from the web\cite{webvision}, it is difficult to collect large amounts of data for each class equally, and the size of the training set for head and tail categories can differ by a factor of hundred or more. 
For such datasets, unconditional diffusion models tend to produce a significant portion of low-quality images. Conditional models, as shown in \Figref{fig:ltdataset}, generate head class images with satisfying performance, while conversely the generated images on tail classes are very likely to show unrecognizable semantics. 
Concerning training generative models with limited data, there already exist several methods\cite{GANlimited, AUGMDifferential, ganregjensen} based on GAN models\cite{brock2018large}. However, quite few studies examine the impact of imbalanced class distribution\cite{harsh2022ganlt} especially on diffusion models, which is practical yet under-explored.

Our work first introduces diffusion models to imbalance generation tasks on several long-tailed datasets\cite{cifar100lt}, and then build some straightforward baselines according to the common methods used in long-tailed recognition\cite{menon2021logitadjustment, mahajan2018sqrtros}. 
To overcome the potential degeneration induced by the skewed distribution, we propose a novel Class-Balancing Diffusion Model (CBDM). 
Theoretically, CBDM resorts to adjusting the conditional transfer probability during sampling in order to implicitly force generated images to have a balanced prior distribution during every sampling step. Technically, the adjusted transfer probability of CBDM results in an additional MSE-form loss for a conditional diffusion model, which functions as a regularizer. 
Intuitively, this loss augments the similarity of generated images conditioned on different classes, and turns out to be an effective approach to transfer common information from head classes to tail classes without hurting the model's expressiveness on head classes. CBDM can be implemented within several flines of codes, and its lighter version admits fine-tuning an existing conditional model. We conducted extensive experiments on CIFAR10/CIFAR100 and their corresponding long-tailed dataset to show the promise of CBDM over existing state-of-the-art methods. 
In a nutshell, the contributions of this work can be summarized as follows:

\begin{itemize}
    \item We identify the severe degeneration problem of diffusion models in long-tailed generation tasks and benchmark some straightforward baselines in this direction.
    \item We propose a new perspective to handle the generation quality collapse on tail classes, and derive a novel Class-Balancing Diffusion Model, which is effective and lightweight as a regularizer to existing methods.
    \item We validate that CBDM is capable of generating more diverse images with convincing fidelity, especially for datasets with large number of categories. In addition, CBDM is robust with accelerating algorithms such as DDIM\cite{ddim}, and can be transplanted to different conditional diffusion-based backbones easily.
\end{itemize}

%% file: section2.tex
\section{Related Works}

\begin{figure*}[ht]
    \centering
    \begin{subfigure}[b]{0.7\linewidth}
         \centering
         \includegraphics[width=0.99\linewidth]{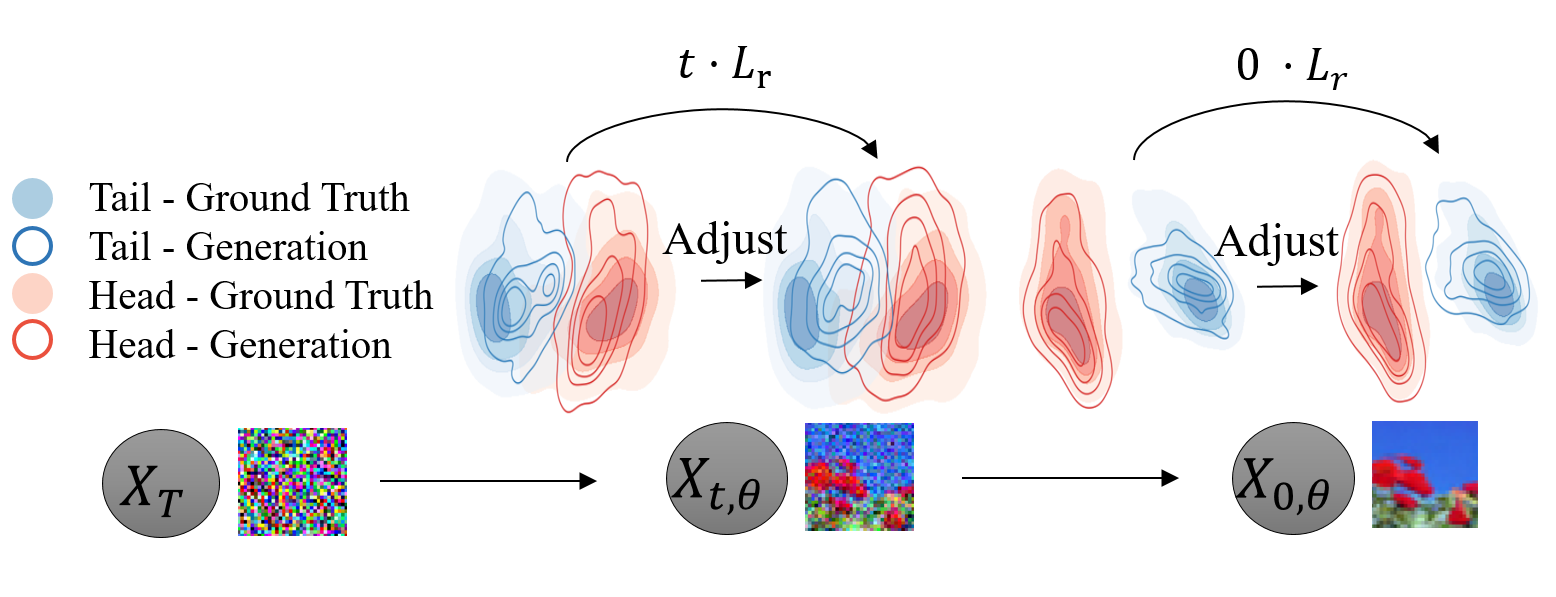}
         \caption{Principle overview of CBDM}
         \label{fig:observation-distribution}
    \end{subfigure}
    \hfill
    \begin{subfigure}[b]{0.29\linewidth}
         \centering
         \includegraphics[width=0.99\linewidth]{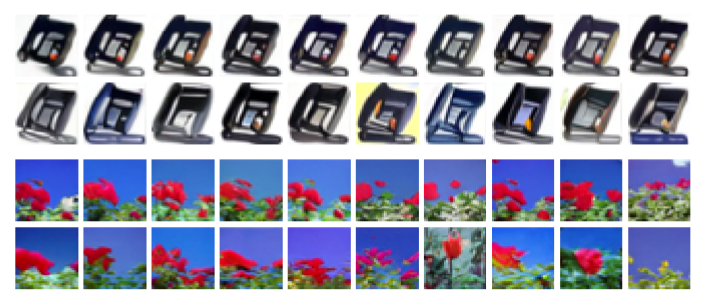}
         \caption{
         DDPM(top)/CBDM(bottom) comparison when denoising a same noised image in class 70/86. We use both generators to recover a noisy image of two classes and observe that CBDM is able to produce more diversity based on the same starting point.
        }
        \label{fig:observation-generation}
    \end{subfigure}
    \caption{Algorithm (\textit{left}) and generation (\textit{right}) visualization. In the \textit{left} figure, we show that an extra regularization loss $gL_r$ proportional to the diffusion step $t$ is added during training. This loss function pushes the sampling distribution (curves on the surface) to wider region while preventing it to be excessively distorted compared to the ground truth distribution (gradient color on the background). 
    }
    \label{fig:observation}
\end{figure*}

\paragraph{Diffusion models}
Diffusion models are recently proposed generative models\cite{nonequilibrium} based on non-equilibrium thermodynamics. 
Conditional diffusion models\cite{guided} encode label information into the generation process and improve largely the generation performance. The guidance structure proposed in \cite{guided} makes it possible to control the generation process through an external module. Based on a similar intuition, researchers arrive to realize diverse functions, such as guided adversarial purification \cite{wu2022guidedpurification}, few-shot generation\cite{giannone2022fewshot} and so on \cite{vikash2022lowdensity, rombach2022highreso}.
The drawback of classifier guidance (noted as CG)\cite{guided} lies in its requirement of training another auxiliary classifier. To address the issue, classifier-free guidance (noted as CFG)\cite{classifier-free} proposed a mechanism that uses the generator itself to express the class guidance information. CFG is proved to be a resource efficient method and achieves outstanding performance on large models\cite{glide}. Moreover, CFG only requires to add one line in training, which can be easily transplanted on different models.

\vspace{-13pt}
\paragraph{Long-tail recognition}
The problem of long-tailed distribution is a common dilemma in machine learning and have been widely explored in the area of discriminative models, \textit{i.e.,} long-tailed recognition. There mainly exist three paradigms in this domain, namely Class Re-balancing\cite{mahajan2018sqrtros, menon2021logitadjustment}, Information Augmentation\cite{yin2019ftl, li2021leapl, wang2019rsg} and Module Improvement\cite{lmle, hfl}. Among them, Class Re-balancing provides the best explainability, and its most common practice is re-sampling\cite{mahajan2018sqrtros}. Thereafter, stemming from modifying the objective function from the global error rate to the class average one, \cite{menon2021logitadjustment} propose logit adjustment which has shown an impressive performance. 
Another effective method is based on Information Augmentation, which uses head class feature information to augment tail classes \cite{li2021leapl, chu2020augmlt, wang2019rsg, yin2019ftl}.
However, discriminative models map data from higher to lower dimensions, while generative models map images from lower to higher dimensions. Thus, the mechanism of class rebalancing between them may be completely different and how to design the balancing method remains under-explored.

\vspace{-13pt} 
\paragraph{Generative models based on limited data} 
GAN\cite{brock2018large, deshpande2018generative} is the most dominant model in the field of image synthesis in recent years. Given the requirement of large-scale data to train the generative models, a part of researchers focus on improving the performance of GAN models under the small datasets. To address the overfitting issue of the discriminator, a number of regularization methods \cite{ ganregjensen,ganregconsis, 2021ganreglimited} has been proposed. An alternative solution is Data Augmentation. As augmentation information is very prone to leak to the generator, researchers proposed improved augmentation strategies such as Differentiable Augmentation (DiffAug)\cite{AUGMDifferential} and Adaptive Augmentation (ADA)\cite{GANlimited}. 
These augmentation ways can be transferred to diffusion models as baselines to help the generation of tail-class samples.

%% file: section3.tex
\section{Method}
In this section, we first give the basic notations following the classical DDPM model, and then introduce the class-imbalanced generation setting. In the third part, we introduce our CBDM algorithm and present its training details including implementation and hyper-parameter settings.

\subsection{Preliminary}
Diffusion models leverage a pre-defined forward process in training, where a clean image distribution $q(\rvx_0)$ can be corrupted to a noisy distribution $q(\rvx_t|\rvx_0)$ at a specified timestep $t$. 
Given a pre-defined variance schedule $\{\beta_t\}_{1:T}$, the noisy distribution at any intermediate timestep is 
$$q(\rvx_t \given \rvx_0) = \gN(\sqrt{\bar \alpha_t} \rvx_0, (1-\bar \alpha_t)\mI);~~~ \bar{\alpha}_t\!\!=\!\!\prod_{i=1}^t (1-\beta_i).$$
To reverse such forward process, a generative model $\theta$ learns to estimate the analytical true posterior in order to recover $\rvx_{t-1}$ from $\rvx_t$ as follows:
$$\min_\theta \KL[q(\rvx_{t-1}|\rvx_t, \rvx_0)||p_{\theta}(\rvx_{t-1}|\rvx_t)];~~\forall t \in \{1, ..., T\},$$
and such an objective can be reduced to a simple denoising estimation loss\cite{ddpm}:
\begin{equation}\label{eq:mse}
\resizebox{\columnwidth}{!}{$\gL_\text{DDPM} \!=\! \E_{t, \rvx_0 \sim q(\rvx_0), \rvepsilon \sim \gN(\vzero, \mI)} \left[ \|\rvepsilon - \rvepsilon_{\theta}(\sqrt{\bar{\alpha}}\rvx_0+\sqrt{1-\bar{\alpha_t}}\rvepsilon, t)\|^2 \right]$}
\end{equation}
For the case where label information is available, the model is trained to estimate the noise as above in both conditional cases $\rvepsilon_{\theta}(\rvx_t, y, t)$ with data-label pairs $(\rvx_0,y)$ and unconditional case $\rvepsilon_{\theta}(\rvx_t, t)$. In the sampling, the label-guided model estimates the noise with a linear interpolation $\hat \rvepsilon = (1+\omega)\rvepsilon_{\theta}(\rvx_t, y, t) - \omega \rvepsilon_{\theta}(\rvx_t, t)$ to recover $\rvx_{t-1}$, which is often referred as Classifier-Free Guidance (CFG)\cite{classifier-free}.

\subsection{Class-Balancing Diffusion Models}
Current diffusion models assume the data distribution to be uniform in class, and thus equally treat samples in the training stage. However, based on our observation, such training strategy leads to degradation in generation quality.
Below, we provide an analysis that motivates our Class-Balancing Diffusion Models (CBDM). 

{Suppose $q(\rvx, y)$ is the data distribution that we need to match with the joint distribution $p_\theta(\rvx, y)$ predicted by a generative model. }
We analyze their difference from the density ratio $r = \frac{q(\rvx, y)}{p_\theta(\rvx, y)} = \frac{q(\rvx \given y)}{p_\theta(\rvx \given y)} \cdot \frac{q(y)}{p_\theta(y)}$. 
When the true label distribution $q(y)$ is the same as the prior $p_\theta(y)$, which is usually assumed to be uniform, the density ratio $r$ is reduced the conditional term to learn a conditional model $p_\theta(\rvx \given y)$. However, when such a hypothesis is violated, for head classes, $\frac{q(y)}{p_\theta(y)}$ would result in a larger weight that makes the model biased and hurt tail classes, and vice versa. 
Empirically, we observe that the generation degrades more on tail classes, as illustrated in \Figref{fig:ltdataset}. Moreover, as shown in \Figref{fig:observation-distribution}, compared to head classes, DDPM cannot well capture the tail-class data distribution and the mode is poorly covered during the sampling process.
As a result, generations of tail classes often have poor quality and diversity, shown in \Figref{fig:observation-generation}. 

To tackle this issue, the most intuitive approach lies in adjusting the prior label distribution through a class balanced re-sampling. However, such abrupt adjustment easily leads to negative improvement in experiments. The step-by-step sampling nature of diffusion models provides another aspect to adjust this distribution more softly.
In this spirit, we propose to calibrate the learning process through the conditional transfer probability $p_\theta\left(\rvx_{t-1}\middle|\rvx_t,y\right)$ when there exists a gap between the class distribution and the prior. 

Let $p_\theta^{\star}\left(\rvx_{t-1}\middle|\rvx_t,y\right)$ be the optimum trained in the case that $\frac{q(y)}{p_\theta(y)}$ is correctly estimated, and
$p_\theta\left(\rvx_{t-1}\middle|\rvx_t,y\right)$ be the one trained in a class-imbalanced case. The relation between such two generative distributions can be described as the following proposition.

\begin{prop}
When training a diffusion model parameterized with $\theta$ on a class-imbalanced dataset, 
its conditional reverse distribution $p_\theta\left(\rvx_{t-1}\middle|\rvx_t,y\right)$ can be corrected with an adjustment schema:
    \begin{align}
        p_{\theta}^{\star}(\rvx_{t-1}|\rvx_{t}, y) = p_{\theta}(\rvx_{t-1}|\rvx_{t}, y) \frac{p_{\theta}(\rvx_{t-1})}{p_{\theta}^{\star}(\rvx_{t-1})}
        \frac{q^{\star}(\rvx_{t})}{q(\rvx_{t})}
    \end{align}
    \label{theo:adj}
\end{prop}

The proposition above shows that, when trained on a class-imbalanced dataset, a diffusion model can still approach the true data distribution by applying a distribution adjustment schema 
$\frac{p_{\theta}(\rvx_{t-1})}{p_{\theta}^{\star}(\rvx_{t-1})}\frac{q^{\star}(\rvx_{t})}{q(\rvx_{t})}$
at every reverse step $t$. 
However, approximating this schema is not feasible at every sampling step, so CBDM incorporates it into the training loss function to achieve an equivalent objective,
and thus gets rid of the model-free part $\frac{q^{\star}(\rvx_{t})}{q(\rvx_{t})}$. 
By further decomposing $p_{\theta}(\rvx_{t-1})$ and $p^\star_{\theta}(\rvx_{t-1})$ to the expectation of the conditional probability $p^\star_{\theta}(\rvx_{t-1}|\rvx_{t:T}, y)$, we present a upper bound to approximate this probability in Proposition \ref{prop:loss}.

\begin{prop}
For the adjusted loss $\gL_{DM}^\star = \sum_{t=1}^T \gL_{t-1}^\star$, an upper-bound of the target training objective to calibrate at timestep $t$ (i.e. $\gL_{t-1}^\star$) can be derived as:
\begin{align}
\sum_{t\geq1}\gL_{t-1}^{\star}& = \sum_{t\geq1}\KL[q(\rvx_{t-1}| \rvx_t, \rvx_0)\ ||\ p_{\theta}^\star(\rvx_{t-1}|\rvx_{t}, y)] \notag \\
&\leq\sum_{t\geq1} (\underbrace{\KL[q(\rvx_{t-1}| \rvx_t, \rvx_0)\ ||\ p_{\theta}(\rvx_{t-1}|\rvx_{t}, y)]}_{\text{Diffusion model loss } \gL_\text{DM}} \notag\\ 
& + \underbrace{t  \E_{y'\sim q^\star_\mathcal{Y}}[ \KL[p_{\theta}(\rvx_{t-1}|\rvx_{t})||p_{\theta}(\rvx_{t-1}|\rvx_{t}, y')]}_{\text{Distribution adjustment loss } \gL_\text{r}} ]) \notag,
\end{align}
\label{prop:loss}
\end{prop}
\noindent 

The upper bound in the above proposition can be considered as two parts. The first term $\gL_\text{DM}$ corresponds to an ordinary DDPM loss \cite{ddpm} \textit{e.g.,} \eqref{eq:mse}, and the second loss $\gL_\text{r}$ is used to adjust the distribution as a regularization term. Roughly speaking, $\gL_\text{r}$ increases the similarity between the model's output and a random target class. Thus, it reduces the risk of overfitting on the head classes, and enlarges the generation diversity for tail class through knowledge obtained from other classes. When $q^\star_\mathcal{Y}$ is less longtailed than the dataset, this loss also increases the probability for underrepresented tail samples to be chosen during training.

\subsection{Training algorithm} 

\begin{algorithm}[h]
\caption{Training algorithm of CBDM.}\label{xt1_algo}
\begin{algorithmic}[1]
    \For {Every batch of size N}
        \For {$(\rvx_0^{(i)}, y^{(i)})$ in this batch}
            \State Sample $\rvepsilon^{(i)} \sim \mathcal{N}(\vzero, \mI)$, $t\sim \mathcal{U}(\{0, 1, ..., T\})$
            \State $\rvx^{(i)}_{t} = \sqrt{\bar{\alpha}_{t}}\rvx^{(i)}_0+ \sqrt{1-\bar{\alpha}_{t}}\rvepsilon^{(i)}$
            \State Calculate $ \gL_\text{DM} = \| \rvepsilon^{(i)} - \rvepsilon_{\theta}(\rvx^{(i)}_{t}, y^{(i)})\|^2$
            \State Sample $y'^{(i)}$ from $q^\star_\mathcal{Y}$
            \State Calculate the first regularization term
            \\ \quad\quad\quad$\gL_{r} = t\tau ||\rvepsilon_{\theta}(\rvx_{t}^{(i)}, y^{(i)}) - \text{sg}(\rvepsilon_{\theta}(\rvx_{t}^{(i)}, y'^{(i)}))||^2$
            \State Calculate the regularization commitment term \\ \quad\quad\quad$\gL_{rc} = t\tau ||\text{sg}(\rvepsilon_{\theta}(\rvx_{t}^{(i)}, y^{(i)})) - \rvepsilon_{\theta}(\rvx_{t}^{(i)}, y'^{(i)})||^2$
            \State Update with $\gL_\text{CBDM} = \gL_\text{DM} + \gL_{r} + \gamma \gL_{rc}$
        \EndFor
    \EndFor
\end{algorithmic}
\end{algorithm}

The detailed training algorithm of CBDM is presented in Alg.~(\ref{xt1_algo}). In the algorithm, we reduce the distribution adjustment loss $\gL_{r}$ as a square error loss with Monte-Carlo samples as indicated by Eqn.(\ref{eq:loss-distribution-adjustment}), where $\gY$ is a set of samples that drawn following the distribution $q^\star_\mathcal{Y}$ and $y$ denotes the image label. Note for CFG\cite{classifier-free}, there is a fixed probability (usually 10\%) to drop the condition, \textit{i.e.,} $y=\mathrm{None}$. 

\begin{equation}
    \gL_\mathrm{{r}} (\rvx_t, y, t) = \frac{1}{|\mathcal{Y}|}\sum_{y'\in\mathcal{Y}}[t||\rvepsilon_\theta(\rvx_t, y) - \rvepsilon_\theta(\rvx_t, y')||^2], \label{eq:loss-distribution-adjustment}
\end{equation}

For the implementation in practice, CBDM can be plugged into any existing {conditional} diffusion models by adopting their model architecture and adjusting the training loss $\gL_\text{DM}$ following lines $6-11$.
Specifically, the choice of the regularization weight $\tau$ affects the sharpness of the density ratio $\frac{p_\theta\left(\rvx_t\right)}{p_\theta^{\star}\left(\rvx_t\right)}$. For the theoretical analysis, please refer to our proof of Prop.~(\ref{prop:loss}) in the Appendix. 
In addition, the choice of the sampling set $\mathcal{Y}$ is another important perspective of CBDM, which depends on the target distribution we wish to adjust. Without loss of generality, we discuss two cases here. On the one hand, we can adjust the label distribution to a class-balanced label distribution, where we sample labels to construct $\mathcal{Y}^{bal}$. On the other hand, if the data distribution is heavily long-tailed, we can also target the adjusted distribution to a relatively less class-imbalanced distribution for stabilized training.
In our experiments, we show CBDM can work well for both cases in different mechanisms. 

Moreover, we observe that naively optimizing with this loss could make the model collapse to some trivial solutions, where the model outputs the same result regardless of the condition $y$ and thus degenerates the conditional generation performance. Therefore, we follow previous works to apply a stop gradient operation \cite{vqvae,chen2021exploring} to prevent this issue.
The final loss of CBDM is
\begin{align}
    \gL_\mathrm{CBDM}^{\tau, \gamma, \gY} (\rvx_t, y, t, \rvepsilon) = &\|\rvepsilon_\theta(\rvx_t, y) - \rvepsilon \|^2 \notag\\ 
    + \frac{\tau t}{|\gY|} \sum_{y'\in \mathcal{Y}} ( &\|\rvepsilon_\theta(\rvx_t, y) - \text{sg}(\rvepsilon_\theta(\rvx_t, y')) \|^2  \notag \\
    + \gamma &\|\text{sg}(\rvepsilon_\theta(\rvx_t, y)) - \rvepsilon_\theta(\rvx_t, y') \|^2), \label{eq:loss-CBDM}
\end{align}
where ``sg($\cdot$)" denotes the stop gradient operation; $\tau$, $\gamma$ are weights for the regularization and the commitment term respectively with $\gamma$ set to $\frac{1}{4}$ in a default setting, and $|\gY|$ denotes the number of elements in the label set.

%% file: section4.tex
\section{Experimental results}

\begin{figure}[t]
    \centering
    \begin{subfigure}[b]{0.49\linewidth}
         \centering
         \includegraphics[width=0.99\linewidth]{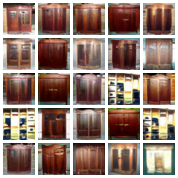}
         \caption{DDPM / Tail Class 94}
         \label{fig:94-free}
    \end{subfigure}
    \hfill
    \begin{subfigure}[b]{0.49\linewidth}
         \centering
         \includegraphics[width=0.99\linewidth]{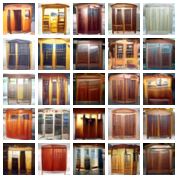}
         \caption{CBDM / Tail Class 94}
         \label{fig:94-reg}
    \end{subfigure}
    \hfill
    \begin{subfigure}[b]{0.49\linewidth}
         \centering
         \includegraphics[width=0.99\linewidth]{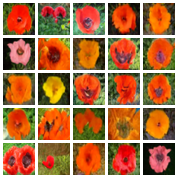}
         \caption{DDPM / Body Class 62}
         \label{fig:62-free}
    \end{subfigure}
    \hfill
    \begin{subfigure}[b]{0.49\linewidth}
         \centering
         \includegraphics[width=0.99\linewidth]{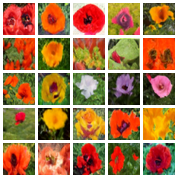}
         \caption{CBDM / Body Class 62}
         \label{fig:62-reg}
    \end{subfigure}
    \hfill
    \caption{Comparison of image generation on heavily tail-distributed (94) and mild tail-distributed (62) classes between DDPM and CBDM.}
    \label{fig:visual}
\end{figure}

\subsection{Experimental setup}
\label{subsec:setup}

\paragraph{Datasets}
We first chose two common datasets that are widely used in the domain of image synthesis, CIFAR10/CIFAR100, and their corresponding long-tailed versions CIFAR10LT and CIFAR100LT. The construction of CIFAR10LT and CIFAR100LT follows \cite{cao2019learning}, where the size decreases exponentially with its category index according to the imbalance factor $\textit{imb}=0.01$. We also conduct experiments on 3 higher-resolution datasets, whose details and results could be found in the Appendix.
\vspace{-12pt} \paragraph{Implementation details} We strictly follow the training configurations of the baseline models. {For DDPM, we set diffusion schedule $\beta_1={10}^{-4}$ and $\beta_T=0.02$ with $T=1,000$, and optimize the network with an Adam optimizer whose learning rate is $0.0002$ after 5,000 epochs of warmup. Considering that the size and semantic complexity of the datasets vary greatly, we choose appropriate epochs for each dataset}. 
In the ablation studies, we also include EDM \cite{kdiffusion}
and Diffusion-ViT \cite{diffusevit} as our backbones to show the compatibility with different architectures and score-matching methods.
Note that, when a model does not use conditional input, we follow \cite{guided} to slightly modify the backbone
~\cite{vit} by adding an extra embedding layer. 
\vspace{-12pt} \paragraph{Baselines}
We first adopt following classic methods to build baselines: re-sampling (RS), soft re-sampling methods (RS-SQRT) \cite{mahajan2018sqrtros}, and augmentation-based methods including Differentiable Augmentation (DiffAug) \cite{AUGMDifferential} and Adaptive Augmentation (ADA) \cite{GANlimited}.
Besides diffusion-based baselines, previous state-of-the-art generative models that study 
the long-tailed distribution, \textit{e.g.,} CBGAN \cite{cbgan}, 
and group spectral regularization for GANs \cite{harsh2022ganlt}, are also included in the comparison. 
Precisely, RS follows the uniform class distribution and RS-SQRT \cite{mahajan2018sqrtros} samples with a probability determined by the square root of class frequency.
And DiffAug \cite{AUGMDifferential} is directly applied to all training images following their default threshold. Following \cite{kdiffusion}, ADA~\cite{GANlimited} is not only applied to images, the augmentation pipeline is also encoded as a condition via an additional embedding layer in the U-Net.
\vspace{-12pt} \paragraph{Metrics}
CBDM and corresponding baselines are evaluated in terms of both generation diversity and fidelity via Frechet Inception Distance (FID) \cite{fid}, Inception Score (IS) \cite{is}, Recall \cite{kyn2019prd} and ${F}_{\beta}$ \cite{sajjadi2018prd}. {We measure Recall and ${F}_{\beta}$ using Inception-V3 features, and take respectively $K=5$ for Recall, 1/8 and 8 for the threshold in ${F}_{\beta}$, and 20 times of class number as the clustering number of ${F}_{\beta}$ to capture the inner class variance}. To this end, Recall and ${F}_{8}$ can be regarded as diversity metrics and IS and ${F}_{1/8}$ are more inclined to measure fidelity. In evaluation, we take the class-balanced version corresponding to those used in training, and all metrics are measured with 50k generated images (10k for ablation experiments). The classifier-free guidance is adapted in sampling and we tune the guidance strength $\omega$ for both baselines and our method to ensure the best performance. 
Specifically, we take $\omega = 1.6, 0.8, 1.0, 0.8$ respectively for the four CIFAR datasets mentioned above.

\begin{table*}[t]
  \centering
  \begin{tabular}{c|l|lllll}
    \toprule[1.5pt]
    Dataset & Model & \textbf{FID}$\downarrow$ &  $\bm{F_{8}}$ $\uparrow$&\textbf{Recall}$\uparrow$ & IS$\uparrow$  & $F_{1/8}$ $\uparrow$ \\
    \midrule
    CIFAR100LT & DDPM\cite{ddpm} & 7.38	& 0.85  & 0.52   &13.11&  0.88 \\
    &  ~+ADA\cite{GANlimited} & 6.16 & 0.91 &  0.57   & 12.71 & 0.90 \\
    &  ~+DiffAug\cite{AUGMDifferential}  & 9.19  & 0.88  &  0.47   & 11.56 & 0.86 \\    
    & ~+ RS\cite{mahajan2018sqrtros} & 10.50 & 0.65    & 0.49   & 12.60 & 0.83 \\
    & ~+SQRT- RS\cite{mahajan2018sqrtros} & 9.72  & 0.66  & 0.47    & \textbf{13.47} & 0.83 \\ 
    \rowcolor{Gray}
    \cellcolor{white}& CBDM (ours)  & 6.26 $_{\textcolor{blue}{(-1.12)}}$ & 0.91 $_{\textcolor{blue}{(+0.06)}}$&   0.57 $_{\textcolor{blue}{(+0.05)}}$ &   13.24 $_{\textcolor{blue}{(+0.13)}}$ & 0.89 $_{\textcolor{blue}{(+0.01)}}$ \\
    \rowcolor{Gray}
    \cellcolor{white}& ~+ADA\cite{GANlimited}  & \textbf{5.81} $_{\textcolor{blue}{(-1.57)}}$  & \textbf{0.91} $_{\textcolor{blue}{(+0.06)}}$ & \textbf{0.57} $_{\textcolor{blue}{(+0.05)}}$  & 13.34 $_{\textcolor{blue}{(+0.23)}}$ & \textbf{0.90} $_{\textcolor{blue}{(+0.02)}}$ \\
    \midrule
    CIFAR100& DDPM  & 3.11	& 0.97 &  0.65   & 13.65  & 0.96 \\
    \rowcolor{Gray}
    \cellcolor{white}& CBDM (ours)  & \textbf{2.72} $_{\textcolor{blue}{(-0.39)}}$ & 0.97 $_{{(\pm 0)}}$& \textbf{0.67 } $_{\textcolor{blue}{(+0.02)}}$   & \textbf{14.03} $_{\textcolor{blue}{(+0.38)}}$ &  {0.96 }$_{{(\pm 0)}}$\\
    \midrule
    CIFAR10LT & DDPM & 5.76& 0.97 & 0.57    &9.17 & 0.95 \\
    \rowcolor{Gray}
    \cellcolor{white}& CBDM (ours)  &\textbf{5.46} $_{\textcolor{blue}{(-0.30)}}$ & {0.97 }$_{{(\pm 0)}}$ &  \textbf{0.59 }$_{\textcolor{blue}{(+0.02)}}$ &  \textbf{9.28} $_{\textcolor{blue}{(+0.11)}}$&  {0.95 } $_{{(\pm 0)}}$\\
    \midrule
    CIFAR10 & DDPM &  3.16	& 0.99 &   0.64  & \textbf{9.80} & {0.98 } \\
    \rowcolor{Gray}
    \cellcolor{white}&CBDM (ours) &  \textbf{3.03} $_{\textcolor{blue}{(-0.13)}}$ &  {0.99 }$_{{(\pm 0)}}$ &  \textbf{0.65 }$_{\textcolor{blue}{(+0.01)}}$  & 9.63 $_{\textcolor{red}{(-0.17)}}$ &  0.98 $_{{(\pm 0)}}$   \\
    \bottomrule[1.5pt]
    \end{tabular}
    \caption{CBDM performance on different datasets. In the table, the first three columns are diversity-related, and we mark the best results in bold. As the comparison between CBDM and DDPM is more straightforward, we mark the performance gain in parentheses next to the CBDM results, using {blue} and {red} to indicate {improvements} and {degradation} respectively.}
  \label{tab:res}
\end{table*}

\subsection{Main results}
In this part, we consider DDPM as the most direct baseline. As shown in \Tabref{tab:res}, CBDM overall outperforms DDPM on all datasets {except the IS is slightly lower on CIFAR10}. 
For datasets with more classes, the improvements are more significant in terms of both diversity and fidelity, as on CIFAR100LT and CIFAR100.
We further investigate the performance of diffusion models when combined with augmentation-based and re-sampling methods. Surprisingly, except for ADA, we consistently observe a performance degradation when these methods are compared with the vanilla DDPM. Moreover, adding the ADA augmentation with CBDM training leads to a further improvement in terms of both diversity and fidelity.

As a qualitative justification, \Figref{fig:visual} provides the visualization comparison between DDPM and CBDM on a relatively mild tail-distributed class (62) and a tail-distributed class (94). We remark that CBDM generates more diverse images. For example, on the tail class 94, the cabinets shown in \Figref{fig:94-reg} have more color and texture, which justifies the improvements of diversity-related metrics in \Tabref{tab:res}. On the contrary, DDPM only generates images that are highly similar to the training data. 

\begin{table}
  \centering
  \resizebox{\columnwidth}{!}{
  \begin{tabular}{c|c|c|cccc}
    \toprule
    \textit{PT} & \textit{FT} & $\mathcal{Y}$ & \textbf{FID}  & $\bm{F_8}$  & \textbf{Recall}  & IS \\
    \midrule
    CIFAR
    &  - & $\mathcal{Y}^{train}$& 6.26& 0.91 & 0.57 &\textbf{13.24}\\
    100LT& & $\mathcal{Y}^{bal}$  & 6.10& \textbf{0.94 } & \textbf{0.64} & 12.29 \\
    &  & $\mathcal{Y}^{sqrt}$      & \textbf{6.00} & 0.93 & 0.59 &13.00 \\
    \bottomrule
    CIFAR & CIFAR & - & 7.38& 0.85  & 0.52  & 13.11  \\
    100LT& 100LT& $\mathcal{Y}^{train}$  & 6.20 & 0.88 & 0.54  &\textbf{13.36} \\
    &   & $\mathcal{Y}^{bal}$ & \textbf{5.85} &\textbf{0.92 } & \textbf{0.60 } & 13.29 \\
    \midrule
    CIFAR & CIFAR & -  & 5.39 & 0.66  & 0.58  & 9.43  \\
    100&   10LT& $\mathcal{Y}^{train}$  & \textbf{5.05}& \textbf{0.67}  & 0.60 & 9.48  \\
    &  & $\mathcal{Y}^{bal}$   & 5.90 &0.66  & 0.60  & \textbf{9.49} \\
    \bottomrule
    \end{tabular}}
    \caption{CBDM performance for 3 mechanisms under different regularization sampling set $\mathcal{Y}$. Column \textit{PT} indicates the dataset used in {(pre)training}, column \textit{FT} indicates the dataset used for fine-tuning.
    The rows where $\gY$ marked as ``-`` represent the results with DDPM, {and rows where \textit{FT} marked as ``-`` represent the results of CBDM trained from scratch. Three diversity-preferred metrics and the best results are in bold for emphasis.}}
  \label{tab:finetune}
\end{table}

\begin{figure}[t]
    \centering
     \includegraphics[width=0.9\linewidth]{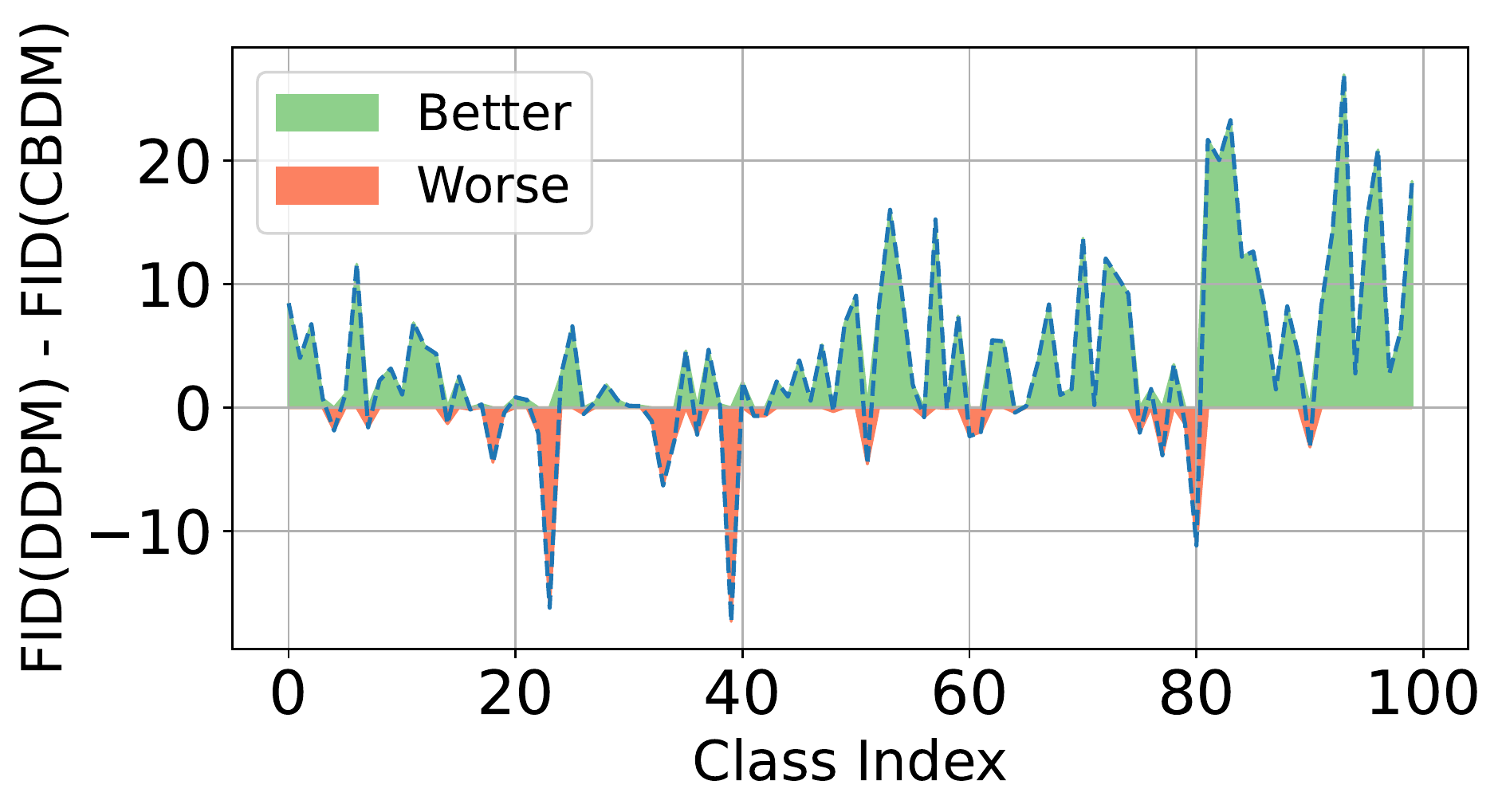}
    \caption{FID improvement per class compared to DDPM. The curve is smoothened by a moving average of 5 unit.}
    \label{fig:classwisefid}
\end{figure}

\begin{table}[t]
    \centering
    \resizebox{\columnwidth}{!}{
      \begin{tabular}{l|l|cc}
        \toprule
        Dataset & Model & FID & IS  \\
        \midrule
        CIFAR & BigGAN+DA+$R_\text{LC}$\cite{2021ganreglimited}$_{{(2021)}}$  & 2.99 & - \\
        100& DDPM  & 3.11 & 13.65\\
        \rowcolor{Gray}
        \cellcolor{white}&CBDM (ours)  & \textbf{2.72} & \textbf{14.03} \\
        \midrule
        CIFAR&CBGAN\cite{cbgan}$_{{(2021)}}$  & 28.17&-\\
        100LT &DDPM   & 12.31 &  12.69\\
        \rowcolor{Gray}
        \cellcolor{white}&CBDM (ours)+ADA\cite{GANlimited}  & \textbf{10.50} & \textbf{12.81} \\
        \midrule
        CIFAR & CBGAN\cite{cbgan}$_{{(2021)}}$  & 32.93&-\\
        10LT& SNGAN\cite{miyato2018spectral}+gSR\cite{harsh2022ganlt}$_{{(2022)}}$ & 18.58& 7.80\\
         & DDPM & 9.68 & 9.00 \\
         \rowcolor{Gray}
         \cellcolor{white}& CBDM (ours) & \textbf{9.38} &9.12 \\
        \bottomrule
        \end{tabular}}
        \caption{Comparision with long-tailed SoTAs on CIFAR. Following their setting, all methods are evaluated with 10k generated images and ground truth images are from their validation dataset. {The publish year of every baselines are marked in next to the citation.}}
      \label{tab:benchmark}
\end{table}

\vspace{-12pt} \paragraph{Case-by-case study.} 
To better understand the generation conditioned on each class, we compare the FID case-by-case between DDPM and CBDM on each class. The results are shown in \Figref{fig:classwisefid}. Compared to DDPM, CBDM shows more consistent improvements for the class index greater than 40. In the tail classes, CBDM outperforms than DDPM the FID by a more significant margin.
\vspace{-12pt} \paragraph{The choice of label set $\mathcal{Y}$}
\label{finetune}

We investigate the effects of label distribution for CBDM under three mechanisms. The first mechanism is the vanilla CBDM given in \Algref{xt1_algo} which trains a diffusion model with regularization from scratch. Besides, we also study CBDM with two fine-tuning configurations. The first setting concerns using CBDM to adjust an existing model pre-trained on a long-tailed dataset to improve its performance. The second one refers to fine-tuning a pre-trained model to adapt to a smaller dataset with an imbalanced class distribution using CBDM. For those mechanisms, we consider three settings for sampling the label set:  (1) a label distribution similar to the training set (denoted $\gY^{train}$); (2) a totally uniform label distribution (denoted $\gY^{bal}$); (3) a less long-tailed distribution compared to $\gY^{train}$, whose class frequency is the square root of the original one (denoted $\gY^{sqrt}$).

The results are shown in \Tabref{tab:finetune}, where we found that CBDM performs well on all these configurations.
As we analyzed in section~\ref{sec:approach}, when the empirical distribution of $\gY$ has a significant difference from the label distribution of the dataset, adjusting with $\gY^{bal}$ could hurt training stability and result in performance degradation. We observe that $\gY^{bal}$ only shows better performance regarding generation diversity, while the generation fidelity has an obvious gap than the other settings. Using a relatively mild set $\gY^{sqrt}$ has better performance in terms of IS and preserves the generation diversity, which leads to the best FID among all these settings. 
On the contrary, when a pre-trained model is available, fine-tuning on top of it ensures better stability, and we can observe in such case using $\gY^{bal}$ has significant improvement than the vanilla DDPM or {fine-tuning using} $\gY^{train}$. When adapting to a different dataset, we also observe using $\gY^{bal}$ causes stability issue and produce undesired results.

\vspace{-12pt}
\paragraph{Enhancement of training classifiers on long-tailed data}
\begin{table}[t]
  \centering
  \begin{tabular}{l| cc}
    \toprule
    Training data & Precison & Recall \\
    \midrule
    \textcolor{gray}{{CIFAR100}} & \textcolor{gray}{0.70}  & \textcolor{gray}{0.67}  \\ \hline
    CIFAR100LT & 0.45	& 0.39 \\
    ~+ DDPM gens (50k)   & 0.48 $_{\textcolor{blue}{(+0.03)}}$  & 0.44 $_{\textcolor{blue}{(+0.05)}}$\\
    \rowcolor{Gray}
    ~+ CBDM gens (50k)   & \textbf{0.49} $_{\textcolor{blue}{(+0.04)}}$  & \textbf{0.47} $_{\textcolor{blue}{(+0.08)}}$\\
    \bottomrule
  \end{tabular}
  \caption{Recognition results of different training data. All configurations are evaluated on the testing set of {normal} CIFAR100. Gains based on CIFAR100LT dataset is noted in blue. }
  \label{tab:recog}
\end{table}
As training models on long-tailed data usually leads to undesired classification results, we investigate whether the generated data could help improve the classifier trained on long-tailed data as a complementary evaluation metrics. Here, we train ResNet-32 models, respectively with CIFAR100, {CIFAR100LT} and {CIFAR100LT} augmented with DDPM and CBDM generations (50k samples). 
\Tabref{tab:recog} shows that the training on long-tailed data results in severe performance degradation. When augmenting the long-tailed data with the generated data, we observe fairly improvements in the trained classifier. Moreover, comparing the results between DDPM and CBDM, we observe that the improvements by using CBDM are more significant, which validates the effectiveness of our approach. Especially, the gain on recall is more remarkable than the gain on precision, which means that the generated images are more diverse than DDPM and justifies the results in \Tabref{tab:res}.

\vspace{-12pt} \paragraph{Comparison with other benchmarks}
We include the representative state-of-the-art long-tailed generative modeling approaches in the comparison, most of which are based on GANs since they suffer more severe problems when the training data is skewed. For a fair comparison, we strictly follow their evaluation setting and the results are reported in \Tabref{tab:benchmark}. From the results, DDPM surpasses the performance of the baselines on all the datasets, and CBDM exhibits even stronger performance. 

\subsection{Ablations}
\begin{table}[t]
  \centering
  \begin{tabular}{c|ccccc}
    \toprule
    Backbone & \textbf{FID} &  $\bm{F_{8}}$ &\textbf{Recall} & IS  & $F_{1/8}$  \\
    \midrule
    EDM\cite{kdiffusion}& 8.64 & 0.82 	& 0.48 & 12.16 & 0.86  \\
    \rowcolor{Gray}
     ~+CBDM (ours) & \textbf{7.97}	&\textbf{0.88}  & \textbf{0.52}  & \textbf{12.01} &\textbf{0.86}   \\
    \midrule
    Diffuse-ViT\cite{diffusevit}& 37.6	& 0.76  & 0.45  & 7.66 &  0.62  \\
    \rowcolor{Gray}
     ~+CBDM (ours)& \textbf{30.0}	& \textbf{0.81}  & \textbf{0.51}  & \textbf{8.00} & \textbf{0.65} \\
    \bottomrule
  \end{tabular}
  \caption{CBDM performance using different backbones on CIFAR100LT dataset. Three diversity-biased metrics.}
  \label{tab:backbone}
\end{table}

\begin{table}[t]
  \centering
  \resizebox{\columnwidth}{!}{
  \begin{tabular}{c|lllll}
    \toprule
    Dataset & \textbf{FID}$\downarrow$ &  $\bm{F_{8}}$ $\uparrow$&\textbf{Recall}$\uparrow$ & IS$\uparrow$  & $F_{1/8}$ $\uparrow$ \\
    \hline
    CIFAR & 3.44  & 0.97  & 0.69   & 13.11 & 0.96  \\
    100 & ${\textcolor{red}{(+0.72)}}$  & ${{(\pm 0)}}$  & ${\textcolor{blue}{(+0.02)}}$   & ${\textcolor{red}{(-0.92)}}$ & ${{(\pm 0)}}$ \\
    \hline
    CIFAR & 6.27	& 0.93  & 0.61 & 12.56 & 0.89  \\
    100LT & ${\textcolor{red}{(+0.01)}}$	& ${\textcolor{blue}{(+0.02)}}$  & ${\textcolor{blue}{(+0.04)}}$ & ${\textcolor{red}{(-0.68)}}$ & ${{(\pm 0)}}$ \\
    \bottomrule
  \end{tabular}
  }
  \caption{Sampling with 100 DDIM steps on CIFAR100 and CIFAR100LT. The difference compared with those using DDPM steps are reported below each result.}
  \label{tab:ddim}
\end{table}

\paragraph{Compatibility with different backbones} 
We first investigate the performance of our method with different diffusion models. Without loss of generality, we adopt EDM~\cite{kdiffusion} and Diffusion-ViT~\cite{diffusevit}, which use a different score-matching method and different denoising network backbone. 
As shown in \Tabref{tab:backbone}, compared with the corresponding baselines, CBDM improves the generation result on CIFAR100LT dataset for both backbones, demonstrating the compatibility with a variety of backbones and the effectiveness of handling long-tailed data.

\vspace{-12pt} \paragraph{CBDM with DDIM sampling}
Apart from DDPM reverse sampling and neural SDE methods, recent deterministic ODE methods such as DDIM~\cite{ddim} are widely used in the generation process of diffusion models. We also conduct experiments to study whether CBDM is compatible with DDIM steps. \Tabref{tab:ddim} shows the comparison of using DDPM and DDIM reverse steps when the generation steps are compressed to $1/10$. We note on normal CIFAR100, FID and IS become slightly worse than using 1000 DDPM steps. On CIFAR100LT, except for IS, all metrics remain almost the same as using DDPM steps.

\vspace{-12pt} \paragraph{Effects of hyperparameters}
We tested the effects of regularization weight $\tau$ in CBDM, which is set in default $\tau=\tau_0=0.001$ so that the weight does not surpass $1$ at all steps. As shown in \Figref{fig:tau}, we search weights across different scales and found that the optimal weight is $\tau=\tau_0$. The results also point out that $\tau$ should not be too large or too small. We also tested the difference between CBDM models trained using commitment loss $\gL_{rc}$ and those trained without it. Our experiments show that the FID score of the models using the commitment loss  (8.30) are significantly lower than those of the models not using it (8.84).

\begin{figure}[t]
    \centering
    \begin{subfigure}[b]{0.49\linewidth}
         \centering
         \includegraphics[width=0.99\linewidth]{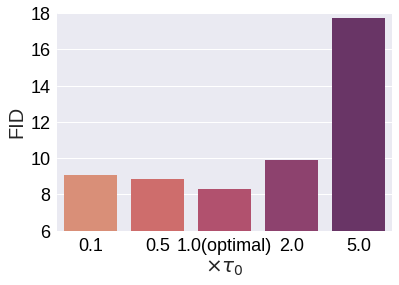}
    \end{subfigure}
    \hfill
    \begin{subfigure}[b]{0.49\linewidth}
         \centering
         \includegraphics[width=0.99\linewidth]{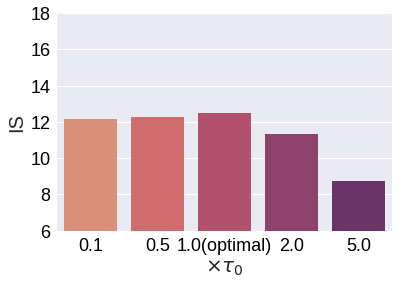}
    \end{subfigure}\vspace{-3mm}
    \caption{FID/IS score under different regularization weight $\tau$}
    \label{fig:tau}
\end{figure}

\vspace{-12pt} \paragraph{Guidance strength $\omega$}
\begin{figure}[t]
    \centering
    \begin{subfigure}[b]{0.49\linewidth}
         \centering
         \includegraphics[width=0.99\linewidth]{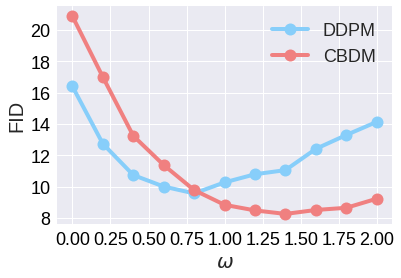}
    \end{subfigure}
    \hfill
    \begin{subfigure}[b]{0.49\linewidth}
         \centering
         \includegraphics[width=0.99\linewidth]{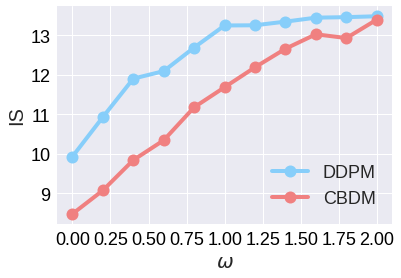}
    \end{subfigure}\vspace{-3mm}
    \caption{FID/IS score under different guidance strength $\omega$}
    \label{fig:omega}
\end{figure}

As we adapt the Classifier-Free Guidance (CFG) during sampling, we also show the effects of guidance strength $\omega$. In \Figref{fig:omega}, we searched $\omega$ from $0$ to $2$ with an interval of $0.2$ and compare FID and IS score of DDPM and CBDM models. We observe that the FID of CBDM remains decreasing at larger guidance strengths compared to DDPM; and although the IS of CBDM is consistently weaker than that of DDPM, they reach the same level when the guidance strength reaches $2.0$, where the FID of CBDM (6.75) is still significantly lower than the best FID of DDPM (7.38).

\vspace{-12pt} \paragraph{Fidelity-diversity control}
\cite{classifier-free} describes the phenomenon where excessive guidance strength $\omega$ tends to lead to better fidelity at the price of overfitting and diversity degradation, as a fidelity-diversity tradeoff. 
CBDM additionally provides another hyperparameter: the regularization weight $\tau$ that promotes model diversity through cross-class information interaction.
\Figref{fig:fidel_div} shows the body class (53) generation results under different guidance strengths and different regularization weight. We observe that a DDPM model with high guidance strength is able to generate very realistic results, but its images overlap almost exactly with a sample in the training set. In contrast, the guidance term of CBDM does not drastically change the image content, but rather refines the image content to be closer to the selected class than the unguided result. This means that CBDM retains more of the diversity of the unguided case and effectively enhances the category information with guide terms.

\begin{figure}[t]
    \centering
     \includegraphics[width=0.95\linewidth]{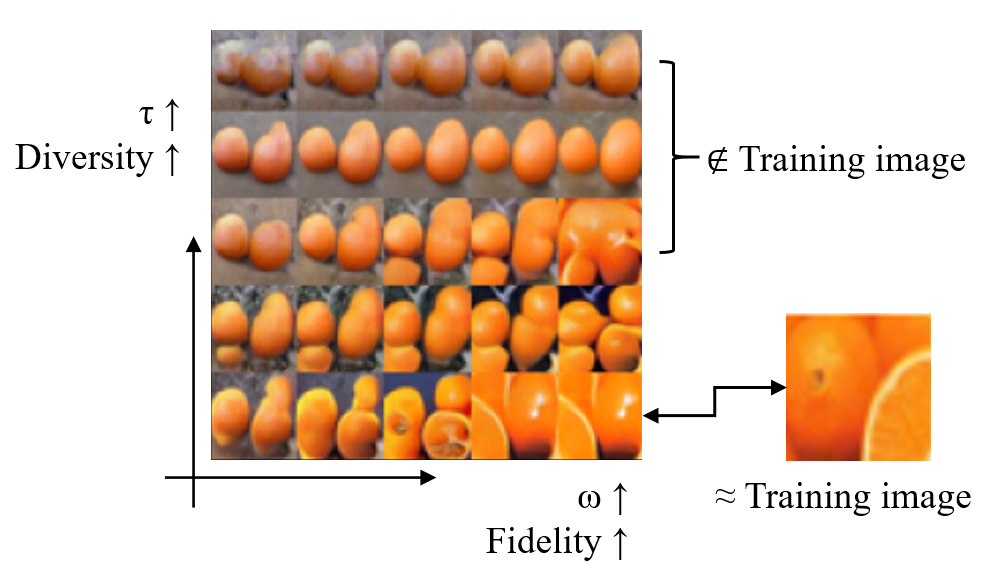}\vspace{-3mm}
    \caption{Image fidelity and diversity controlled by guidance strength $\omega$ and regularization weight $\tau$.}
    \label{fig:fidel_div}
\end{figure}

%% file: PaperSuppMaterials.tex
\newpage
\onecolumn

\begin{center}
    \Large
    \textbf{Supplementary Materials}
\end{center}

\setcounter{prop}{0}

\appendix
\section{Full demonstration for CBDM}

\begin{prop}
When training a diffusion model parameterized with $\theta$ on a class-imbalanced dataset, 
its conditional reverse distribution $p_\theta\left(\rvx_{t-1}\middle|\rvx_{t},y\right)$ can be corrected with an adjustment schema:
    \begin{align}
        p_{\theta}^{\star}(\rvx_{t-1}|\rvx_{t}, y) = p_{\theta}(\rvx_{t-1}|\rvx_{t}, y) \frac{p_{\theta}(\rvx_{t-1})}{p_{\theta}^{\star}(\rvx_{t-1})}
        \frac{q^{\star}(\rvx_{t})}{q(\rvx_{t})}
    \end{align}
    \label{theo:adj}
\end{prop}

\begin{proof}
The starting point for this derivation comes from \cite{guided}. By noting the conditional prior distribution given the image label as $\hat{q}$, we can write the reverse conditional probability $\hat{q}\left(\rvx_{t-1}\middle|\rvx_{t},y\right)$ as

\begin{equation}
    \hat{q}\left(\rvx_{t-1}\middle|\rvx_{t}, y\right)=\frac{q\left(\rvx_{t-1}\middle|\rvx_{t}\right)\hat{q}\left(y\middle|\rvx_{t-1}\right)}{\hat{q}\left(y\middle|\rvx_{t}\right)}
\end{equation}

With the Bayesian formula, the equation can be transformed into

\begin{align}
\hat{q}(\rvx_{t-1}|\rvx_{t}, y) &=\frac{q(\rvx_{t-1}|\rvx_{t})\hat{q}(y|\rvx_{t-1})}{\hat{q}(y|\rvx_{t})}\\
&=\frac{q(\rvx_{t-1}|\rvx_{t})\hat{q}(\rvx_{t-1}|y)\hat{q}(y)\hat{q}(\rvx_{t})}{\hat{q}(\rvx_{t}|y)\hat{q}(y)\hat{q}(\rvx_{t-1})}\\
&=\frac{q(\rvx_{t-1}|\rvx_{t})\hat{q}(\rvx_{t-1}|y)\hat{q}(\rvx_{t})}{\hat{q}(\rvx_{t}|y)\hat{q}(\rvx_{t-1})}
\end{align}

In the above equation, $q(\rvx_{t})$ is equal to $\hat{q}(\rvx_{t})$ according to the derivation of \cite{guided}. Moreover, since the conditional diffusion model is trained to fit a prior distribution with known conditions by definition, we can approximate $\hat{q}(\rvx_{t-1})$ with $p_\theta(\rvx_{t-1})$ in the following calculations. Thus, we have

\begin{align}
p_\theta(\rvx_{t-1}|\rvx_{t}, y)
&=\frac{q(\rvx_{t-1}|\rvx_{t})\hat{q}(\rvx_{t-1}|y)q(\rvx_{t})}{\hat{q}(\rvx_{t}|y)p_\theta(\rvx_{t-1})}
\label{eq:1.1}
\end{align}

We use $\star$ to stand for an optimal distribution, for instance, when the categories are evenly distributed and $\frac{q(y)}{p_\theta(y)}$ is correctly estimated. In this case, we have:

\begin{equation}
    p_{\theta}^{\star}(\rvx_{t-1}|\rvx_{t}, y)  =\frac{q(\rvx_{t-1}|\rvx_{t})\hat{q}^\star(\rvx_{t-1}|y)q^{\star}(\rvx_{t})}{\hat{q}^\star(\rvx_{t}|y)p_{\theta}^{\star}(\rvx_{t-1})}
\end{equation}

Here in \eqref{eq:1.1}, $\hat{q}\left(\rvx_{t}\middle|y\right)$ and $\hat{q}(\rvx_{t-1}|y)$ are conditional on the distribution of the label y and thus are not affected by the label distribution. Then, the only term that is under the influence of imbalanced label distribution is $p_\theta\left(\rvx_{t-1}\right)$ and $q\left(\rvx_{t}\right)$. For the optimal label distribution defined with $\star$, we follow the same deduction. Dividing the two equations yields:

\begin{equation}
    p_{\theta}^{\star}(\rvx_{t-1}|\rvx_{t}, y) = p_{\theta}(\rvx_{t-1}|\rvx_{t}, y) \frac{p_{\theta}(\rvx_{t-1})}{p_{\theta}^{\star}(\rvx_{t-1})}
    \frac{q^{\star}(\rvx_{t})}{q(\rvx_{t})}\label{eq:a.1}
\end{equation}

\end{proof}

\paragraph{From post-hoc adjustment to training with adjustment}
The above \eqref{eq:a.1} shows how to transform the original model estimation $p_{\theta}(\rvx_{t-1}|\rvx_{t}, y)$ into $p_{\theta}^{\star}(\rvx_{t-1}|\rvx_{t}, y)$ by adding the adjustment term  $\frac{p_{\theta}(\rvx_{t-1})}{p_{\theta}^{\star}(\rvx_{t-1})}\frac{q^{\star}(\rvx_{t})}{q(\rvx_{t})}$. 
Naturally, the most direct way to adjust the distribution is to implement this term in a post-hoc manner during sampling. 

First, we note that $\frac{p_\theta(\rvx_{t-1})}{p^{\star}_\theta(\rvx_{t-1})}$ cannot be obtained directly, but can be approximated by the conditional expectation of the model. For example, we can rewrite it as the conditional expectation of $p^{\star}_\theta(\rvx_{t-1}|\rvx_{t}, y)$ successively about the labels $y$ and $\rvx_{t}$, and later estimate it by Monte Carlo sampling.
However, the use of Monte Carlo sampling during image generation imposes much more computational burden, and the image generation will also be inaccurate due to the large estimation error caused by the difficulty to sample uniformly on $\rvx_{t}$ and on $y$. But the most problematic issue is that the adjustment term also contains the true probability of the data distribution $\frac{q^{\star}(\rvx_{t})}{q(\rvx_{t})}$, which obviously cannot be estimated. So the presence of this term makes the post hoc adjustment method impossible to implement. In contrast, following the idea of trainable logit adjustment in the long-tail recognition task\cite{menon2021logitadjustment}, we found that these two problems can be solved if $p_{\theta}$ is adjusted during training.

First, $\frac{q^{\star}(\rvx_{t})}{q(\rvx_{t})}$ is independent of model parameters, and can be thus neglected in the following calculation, which yields a simpler result:

\begin{equation}
    p_{\theta}^{\star}(\rvx_{t-1}|\rvx_{t}, y) = p_{\theta}(\rvx_{t-1}|\rvx_{t}, y) \frac{p_{\theta}(\rvx_{t-1})}{p_{\theta}^{\star}(\rvx_{t-1})}
\end{equation}

Second, similar to the logit adjustment in long-tail recognition, our first step in adjusting the distribution during training lies in reversing the sign before the adjustment term. Since we are considering the gradient of the log probability, changing the positive and negative sign is actually equivalent to reversing the adjustment term $\frac{p_{\theta}(\rvx_{t-1})}{p_{\theta}^{\star}(\rvx_{t-1})}$. Thus, we have instead:

\begin{equation}
    p_{\theta}^{\star}(\rvx_{t-1}|\rvx_{t}, y) = p_{\theta}(\rvx_{t-1}|\rvx_{t}, y) \frac{p_{\theta}^{\star}(\rvx_{t-1})}{p_{\theta}(\rvx_{t-1})}
\end{equation}

\begin{prop}
For the adjusted loss $\gL_{DM}^\star = \sum_{t=1}^T \gL_{t-1}^\star$, an upper-bound of the target training objective to calibrate at timestep $t$ (i.e. $\gL_{t-1}^\star$) can be derived as:
\begin{align}
\sum_{0 < t \leq T}\gL_{t-1}^{\star}& = \sum_{0 < t \leq T}\KL[q(\rvx_{t-1}| \rvx_t, \rvx_0)\ ||\ p_{\theta}^\star(\rvx_{t-1}|\rvx_{t}, y)] \notag \\
&\leq\sum_{0 < t \leq T} (\underbrace{\KL[q(\rvx_{t-1}| \rvx_t, \rvx_0)\ ||\ p_{\theta}(\rvx_{t-1}|\rvx_{t}, y)]}_{\text{Diffusion model loss } \gL_\text{DM}} \notag\\ 
& + \underbrace{t  \E_{y'\sim q^\star_\mathcal{Y}}[ \KL[p_{\theta}(\rvx_{t-1}|\rvx_{t})||p_{\theta}(\rvx_{t-1}|\rvx_{t}, y')]}_{\text{Distribution adjustment loss } \gL_\text{r}} ]) \notag,
\end{align}
\label{prop:loss}
\end{prop}
\noindent where $q^\star_\mathcal{Y}$ is the target label distribution to adjust, \textit{e.g.,} a class-balanced label distribution.

\begin{proof}\label{proof:prop2}
In the proposition, we admit that the adjustment weight (i.e. the regularization weight) is 1. In the deduction, we further denote $\tau$ as the regularization weight, which makes the adjusted probability becomes:
\begin{equation}
    p^{\star}_{\theta}(\rvx_{t-1}|\rvx_{t})=p_{\theta}(\rvx_{t-1}|\rvx_{t}, y)\frac{p_{\theta}^{\star}(\rvx_{t-1})^{\tau}}{p_{\theta}(\rvx_{t-1})^{\tau}}
\end{equation}

Thereafter, by bringing $p^{\star}_{\theta}$ into $L_{t-1}$ of the DDPM and simplifying the formula, we have:
\begin{align}
\sum_{0 < t \leq T}L_{t-1}^{\star}&=\mathbb{E}_q[-\sum_{0 < t \leq T}log\frac{p_{\theta}^{\star}(\rvx_{t-1}|\rvx_{t}, y)}{q(\rvx_{t-1}|\rvx_{t})}] 
\label{eq:2.1}\\
&= \mathbb{E}_q[-\sum_{0 < t \leq T}log\frac{p_{\theta}(\rvx_{t-1}|\rvx_{t}, y)}{q({\rvx_{t-1}|\rvx_{t}})}\frac{p_{\theta}^{\star}(\rvx_{t-1})^{\tau}}{p_{\theta}(\rvx_{t-1})^{\tau}}]\label{eq:2.2}\\
&= \mathbb{E}_q[-\sum_{0 < t \leq T}(log\frac{p_{\theta}(\rvx_{t-1}|\rvx_{t}, y)}{q({\rvx_{t-1}|\rvx_{t}})}-log\frac{p_{\theta}^{\star}(\rvx_{t-1})^{\tau}}{p_{\theta}(\rvx_{t-1})^{\tau}})]\label{eq:2.3}\\
&= \mathbb{E}_q[-\sum_{0 < t \leq T}log\frac{p_{\theta}(\rvx_{t-1}|\rvx_{t}, y)}{q({\rvx_{t-1}|\rvx_{t}})}]+\tau\mathbb{E}_q[-\sum_{0 < t \leq T}log\frac{p_{\theta}^{\star}(\rvx_{t-1})}{p_{\theta}(\rvx_{t-1})}]\label{eq:2.4}
\end{align}

In the above derivation: 
from \eqref{eq:2.2} to equation \eqref{eq:2.3}, we split the product in $log$ into the summation; from  \eqref{eq:2.3} to \eqref{eq:2.4}, the summation is apportioned into two terms, after which the exponent $\tau$ in $log$ is extracted.
We note that the expectation of the divisor of the two probabilities in the log-likelihood is equal to the KL divergence of both, i.e:

\begin{equation}
    \mathbb{E}_q[-log\frac{p_{\theta}(\rvx_{t-1}|\rvx_{t}, y)}{q({\rvx_{t-1}|\rvx_{t}})}] = D_{KL}[q(\rvx_{t-1}|\rvx_{t}, \rvx_{0})\ ||\ p_{\theta}(\rvx_{t-1}|\rvx_{t}, y)]
\end{equation}

Thus, we have:
 
\begin{align}
\sum_{0 < t \leq T}L_{t-1}^{\star}&=\mathbb{E}_q[-\sum_{0 < t \leq T}log\frac{p_{\theta}^{\star}(\rvx_{t-1}|\rvx_{t}, y)}{q(\rvx_{t-1}|\rvx_{t})}]\label{eq:3.1}\\
& = \sum_{0 < t \leq T} D_{KL}[q(\rvx_{t-1}|\rvx_{t}, \rvx_{0})\ ||\ p_{\theta}(\rvx_{t-1}|\rvx_{t}, y)] + \tau\sum_{0 < t \leq T}\mathbb{E}_{q}[-log\frac{p_{\theta}^{\star}(\rvx_{t-1})}{p_{\theta}(\rvx_{t-1})}]\label{eq:3.2}\\
&= \sum_{0 < t \leq T}D_{KL}[q(\rvx_{t-1}|\rvx_{t}, \rvx_{0})\ ||\ p_{\theta}(\rvx_{t-1}|\rvx_{t}, y)] + \tau\sum_{0 < t \leq T}\mathbb{E}_{q}[-\sum_{ t'\geq t}log\frac{p_{\theta}^{\star}(\rvx_{t'-1}|\rvx_{t'})}{p_{\theta}(\rvx_{t'-1}|\rvx_{t'})}]\label{eq:3.3}\\
&\leq \sum_{0 < t \leq T}D_{KL}[q(\rvx_{t-1}|\rvx_{t}, \rvx_{0})\ ||\ p_{\theta}(\rvx_{t-1}|\rvx_{t}, y)] + \tau\sum_{0 < t \leq T}t\mathbb{E}_{q}[-log\frac{p_{\theta}^{\star}(\rvx_{t-1}|\rvx_{t})}{p_{\theta}(\rvx_{t-1}|\rvx_{t})}]\label{eq:3.4}\\
&= \sum_{0 < t \leq T} D_{KL}[q(\rvx_{t-1}|\rvx_{t}, \rvx_{0})\ ||\ p_{\theta}(\rvx_{t-1}|\rvx_{t}, y)] + \tau t D_{KL}(p_{\theta}(\rvx_{t-1}|\rvx_{t})||p_{\theta}^{\star}(\rvx_{t-1}|\rvx_{t}))\label{eq:3.5}\\
&= \sum_{0 < t \leq T} (\underbrace{\KL[q(\rvx_{t-1}| \rvx_t, \rvx_0)\ ||\ p_{\theta}(\rvx_{t-1}|\rvx_{t}, y)]}_{\text{Diffusion model loss } \gL_\text{DM}} + \underbrace{t  \E_{y'\sim q^\star_\mathcal{Y}}[ \KL[p_{\theta}(\rvx_{t-1}|\rvx_{t})||p_{\theta}(\rvx_{t-1}|\rvx_{t}, y')]}_{\text{Distribution adjustment loss } \gL_\text{r}} ])
\end{align}
 
where, from \eqref{eq:3.1} to \eqref{eq:3.2}, we rewrite the first term in the form of KL divergence. From  \eqref{eq:3.2} to  \eqref{eq:3.3}, we decompose $p_{\theta}(\rvx_{t-1})$ into $\mathbb{E}_{\rvx_{t+1:T}}[q(\rvx_{T})\Pi_{t< t' \leq T}p_\theta(\rvx_{t'-1}|\rvx_{t'})]$ in the form of a Markov chain, while we suppose that $\frac{\mathbb{E}_{\rvx_{t+1:T}}[q(\rvx_{T})\Pi_{t< t' \leq T}p_\theta^\star(\rvx_{t'-1}|\rvx_{t'})]}{\mathbb{E}_{\rvx_{t+1:T}}[q(\rvx_{T})\Pi_{t<t' \leq T}p_\theta(\rvx_{t'-1}|\rvx_{t'})]}=\mathbb{E}_{\rvx_{t+1:T}}[\frac{q(\rvx_{T})\Pi_{t< t' \leq T}p_\theta^\star(\rvx_{t'-1}|\rvx_{t'})}{q(\rvx_{T})\Pi_{t<t' \leq T}p_\theta(\rvx_{t'-1}|\rvx_{t'})}]$ to ease calculation. Then the continuous multiplication is converted into the form of a logit sum; thereafter, we shift the summation symbols inside the expectation outward and rewrite the expectation again into the form of KL divergence. For the derivation of  \eqref{eq:3.2} to  \eqref{eq:3.3}, we further simplify this equation through Jensen's inequality as follows:

\begin{align}
    \mathbb{E}_q[-log\frac{p_\theta^{\star}(\rvx_{t-1})}{p_\theta(\rvx_{t-1})}] &\sim \mathbb{E}_q[-log(\mathbb{E}_{\rvx_{t+1:T}}[\frac{q(\rvx_{T})\Pi_{t< t' \leq T}p_\theta^\star(\rvx_{t'-1}|\rvx_{t'})}{q(\rvx_{T})\Pi_{t<t' \leq T}p_\theta(\rvx_{t'-1}|\rvx_{t'})}])\\
    &\leq\mathbb{E}_q[\mathbb{E}_{\rvx_{t+1:T}}[-log\frac{q(\rvx_{T})\Pi_{t< t' \leq T}p_\theta^{\star}(\rvx_{t'-1}|\rvx_{t'})}{q(\rvx_{T})\Pi_{t< t' \leq T}p_\theta(\rvx_{t'-1}|\rvx_{t'})}]] \\
    &=\mathbb{E}_q[-log\frac{\Pi_{t< t' \leq T}p_\theta^{\star}(\rvx_{t'-1}|\rvx_{t'})}{\Pi_{t< t' \leq T}p_\theta(\rvx_{t'-1}|\rvx_{t'})}] \\
    &=\mathbb{E}_q[-\sum_{t< t' \leq T}log\frac{p_\theta^{\star}(\rvx_{t'-1}|\rvx_{t'})}{p_\theta(\rvx_{t'-1}|\rvx_{t'})}] \\    &=\sum_{0\leq t < T}\mathbb{E}_{q_t}[-\sum_{t< t' \leq T}log\frac{p_\theta^{\star}(\rvx_{t'-1}|\rvx_{t'})}{p_\theta(\rvx_{t'-1}|\rvx_{t'})}] \\    &=\mathbb{E}_{q_{t}}[-\sum_{0\leq t < T}\sum_{t< t' \leq T}log\frac{p_\theta^{\star}(\rvx_{t'-1}|\rvx_{t'})}{p_\theta(\rvx_{t'-1}|\rvx_{t'})}] \\
    &=\mathbb{E}_{q}[-\sum_{0< t \leq T}t\cdot log\frac{p_\theta^{\star}(\rvx_{t-1}|\rvx_{t})}{p_\theta(\rvx_{t-1}|\rvx_{t})}] \\
    &=\sum_{0 < t \leq T}t\mathbb{E}_q[-log\frac{p_\theta^{\star}(\rvx_{t-1}|\rvx_{t})}{p_\theta(\rvx_{t-1}|\rvx_{t})}] \\
\end{align}

When $\tau=1$, we obtain the proposition itself; else, we obtain the CBDM algorithm with the hyper-parameter regularization weight $\tau$.
\end{proof}

\paragraph{Loss function}

Based on the above propositions, we derive the final loss form as follows:
\begin{equation}
    \gL_{CBDM}=\gL_{DM}+\gL_{r}
\end{equation}

where:
\begin{eqnarray}
\left\{
\begin{aligned}
	\gL_{DM}(t)& = &D_{KL}(q(\rvx_{t-1}|\rvx_{t}, \rvx_{0})\ ||\ p_{\theta}(\rvx_{t-1}|\rvx_{t},y))\\
	\gL_{r}(t)& = & t  \E_{y'\sim q^\star_\mathcal{Y}}[ \KL[p_{\theta}(\rvx_{t-1}|\rvx_{t},y)||p_{\theta}(\rvx_{t-1}|\rvx_{t}, y')]
\end{aligned}
\right.
\end{eqnarray}

Again, the form of $\gL_{DM}$ is exactly the same as the original loss in DDPM or other common diffusion models, so it is called $\gL_{DM}$. The second term acts similarly to the regularization and is therefore called $\gL_{r}$. Now, we simplify $\gL_{r}$ by first writing $p_\theta^{\star}(\rvx_{t-1}|\rvx_{t})$ as the expectation of the conditional probability about $y'$ realized through a simple Monte-Carlo sampling, and
note the optimal distribution that we want to approximate as $q^\star_\mathcal{Y}$, we have:

\begin{equation}
    \gL_\mathrm{{r}} (\rvx_{t}, y, t) = \E_{y'\sim q^\star_\mathcal{Y}}[t||\rvepsilon_\theta(\rvx_{t}, y) - \rvepsilon_\theta(\rvx_{t}, y')||^2], \label{eq:loss-distribution-adjustment}
\end{equation}

We note that the method can actually be combined with any diffusion model by simply adding the loss $\gL_{r}$ to the DDPM algorithm.
The method can adjust the distribution naturally in the iterations of the training process. Therefore, this method avoids the problem of time-consuming Monte Carlo sampling estimation of the conditional expectation in the post-hoc adjustment method, and it also avoids the problem of estimating the true prior distribution $q^\star(\rvx_{t})$ by directly eliminating terms that are unrelated to the model parameters.

\section{Additional Qualitative Results}

\begin{figure}[t]
    \centering
    \begin{subfigure}[b]{0.49\linewidth}
         \centering
         \includegraphics[width=0.99\linewidth]{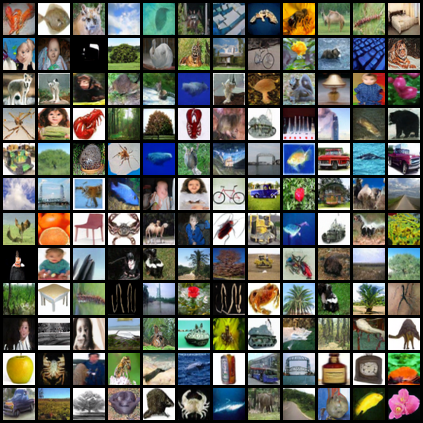}
         \caption{DDPM}
    \end{subfigure}
    \hfill
    \begin{subfigure}[b]{0.49\linewidth}
         \centering
         \includegraphics[width=0.99\linewidth]{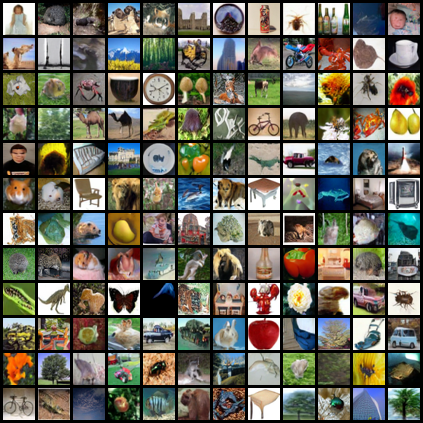}
         \caption{CBDM}
    \end{subfigure}\vspace{-3mm}
    \caption{Image generation visualization for the CIFAR100LT dataset}
    \label{fig:cifar100lt}
\end{figure}

\begin{figure}[t]
    \centering
    \begin{subfigure}[b]{0.49\linewidth}
         \centering
         \includegraphics[width=0.9\linewidth]{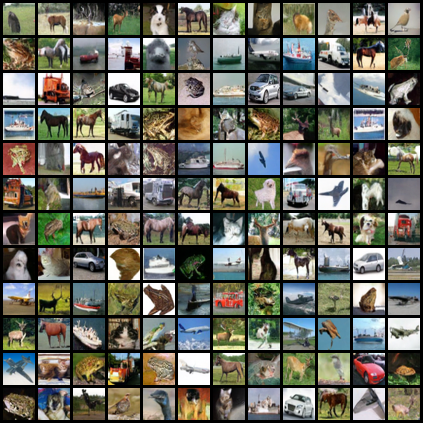}
         \caption{DDPM}
    \end{subfigure}
    \hfill
    \begin{subfigure}[b]{0.49\linewidth}
         \centering
         \includegraphics[width=0.9\linewidth]{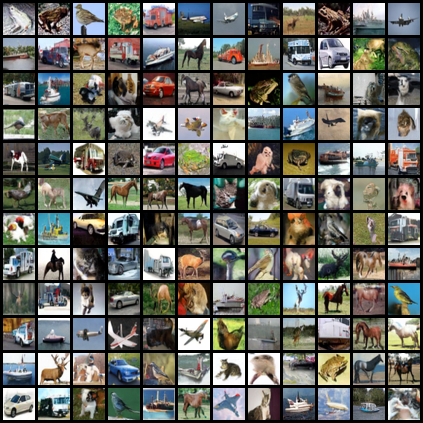}
         \caption{CBDM}
    \end{subfigure}\vspace{-3mm}
    \caption{Image generation visualization for the CIFAR10LT dataset}
    \label{fig:cifar10lt}
\end{figure}

\begin{figure}[h]
    \centering
    \begin{subfigure}[b]{0.53\linewidth}
         \centering
         \includegraphics[width=0.9\linewidth]{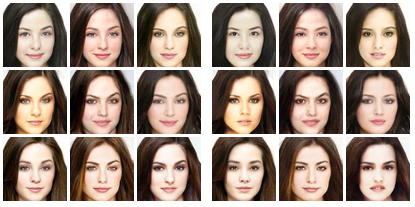}
         \caption{DDPM(left)/CBDM(right) comparison}
    \end{subfigure}
    \hfill
    \begin{subfigure}[b]{0.46\linewidth}
         \centering
         \includegraphics[width=0.95\linewidth]{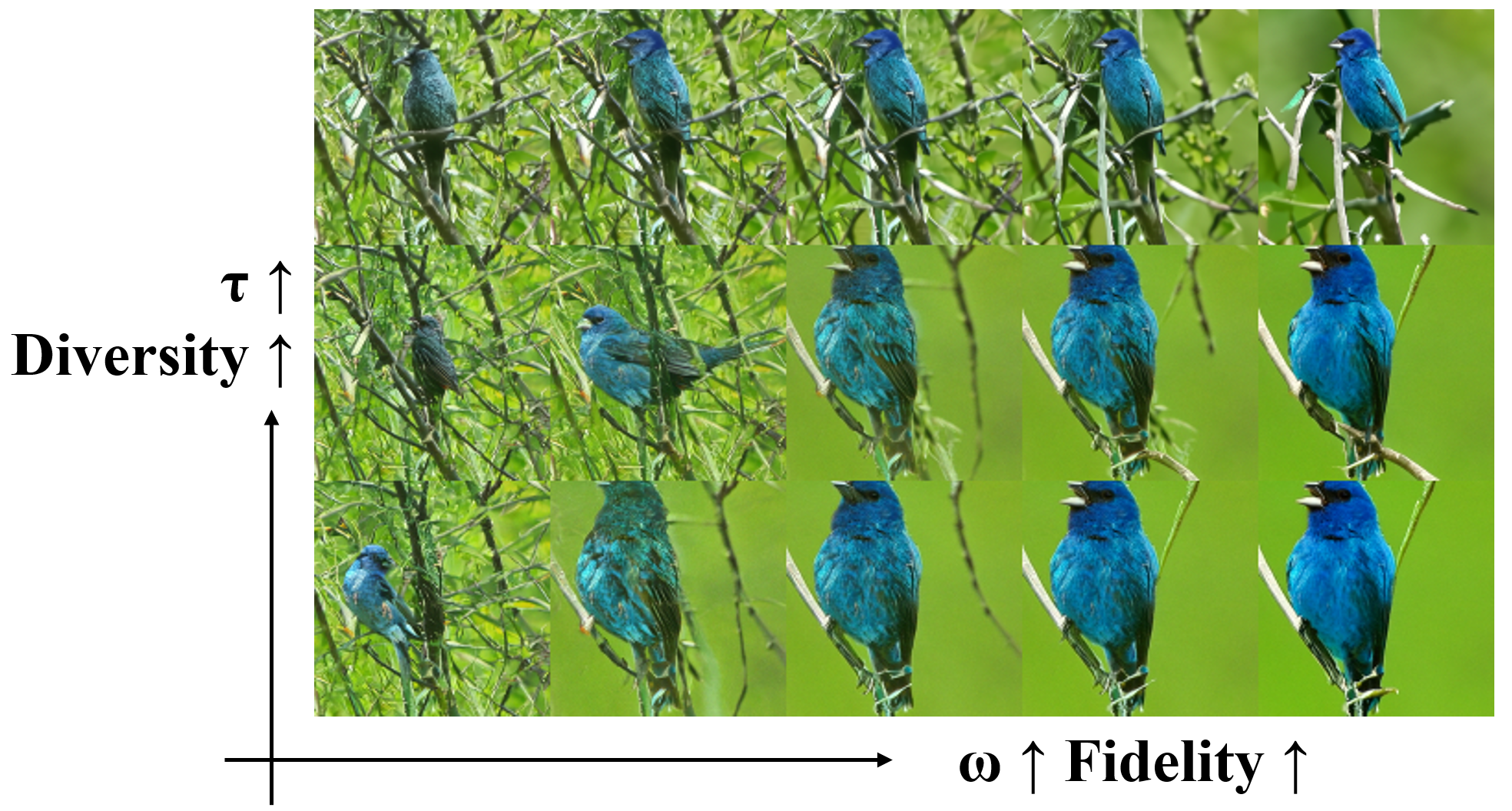}
         \caption{Controllability on CUB dataset}
    \end{subfigure}
    \caption{(a) DDPM(left)/CBDM(right) comparison when denoising a same noised image in CelebA-5.
    (b) An exemplar about the fidelity and diversity control guided by strength $\omega$ and regularization weight $\tau$ for generators trained on the CUB dataset.
    }
    \vspace{-0.3cm}
    \label{fig:newfid}
\end{figure}

\paragraph{Mode collapse caused by an inappropriate $\mathcal{Y}$}

In our experiments, we found that using the theoretically optimal sampling set $\mathcal{Y}^{bal}$ often leads to mode collapse in tail class images, thus we used a more imbalanced label set $\gY$. We visualize the results under different label set $\gY$ in order to better explain this issue. 
In \Figref{fig:sqrtreg}, it can be observed that an imbalanced set $\gY^{train}$ ($\gY^{lt}$) demonstrates a better diversity while preserving the original class information of class 62. On the contrary, $\gY^{sqrt}$ and $\gY^{bal}$ demonstrate a more and more severe mode collapse issue when the set becomes more balanced.

\begin{figure}[H]
    \centering
    \begin{subfigure}[b]{0.24\linewidth}
         \centering
         \includegraphics[width=0.9\linewidth]{62_cifar100lt_free.png}
         \caption{DDPM}
    \end{subfigure}
    \hfill
    \begin{subfigure}[b]{0.24\linewidth}
         \centering
         \includegraphics[width=0.9\linewidth]{62_cifar100lt_reg.png}
         \caption{CBDM($\mathcal{Y}^{train}$)}
    \end{subfigure}
    \hfill
    \begin{subfigure}[b]{0.24\linewidth}
         \centering
         \includegraphics[width=0.9\linewidth]{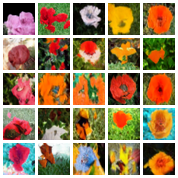}
         \caption{CBDM($\mathcal{Y}^{sqrt}$)}
    \end{subfigure}
    \hfill
    \begin{subfigure}[b]{0.24\linewidth}
         \centering
         \includegraphics[width=0.9\linewidth]{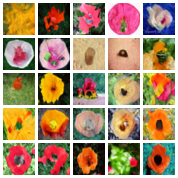}
         \caption{CBDM($\mathcal{Y}^{bal}$)}
    \end{subfigure}
    \caption{Comparison of image generation results on body class (62) between different label set. Image generated by DDPM is also shown for comparison.
    }
    \label{fig:sqrtreg}
\end{figure}

\section{Additional Experimental Results}

\paragraph{Performance on larger datasets}

\begin{table}[h]
    \centering
    \tiny
    \resizebox{0.75\columnwidth}{!}{
      \begin{tabular}{l|l|l|l|cccc}
        \toprule
         Dataset  & Model & {$\omega$} & $\tau$ & \textbf{FID}$\downarrow$ &  $\bm{F_{8}}$ $\uparrow$ & \textbf{Recall} $\uparrow$ & IS$\uparrow$  \\
        \midrule
        CUB\cite{cub} & DDPM & {0.2} & - & 8.34 & 0.91  & 0.70 & 5.34 \\
        \rowcolor{Gray}
        \cellcolor{white}& CB-DDPM (ours) & {0.4} & 0.1 & \textbf{8.23}  & \textbf{0.92} & \textbf{0.70} & \textbf{5.36}  \\
        \midrule
        CelebA-5\cite{Kim_2020_CVPR_celeba5} & DDPM & {0.6} & - & 10.68   & 0.92 & 0.51 & 2.22 \\
        \rowcolor{Gray}
        \cellcolor{white}& CBDM (ours) & {1.0} & 0.05  & \textbf{8.69}  & \textbf{0.94} &\textbf{0.57} & \textbf{2.26} \\
        \midrule
        ImageNet-LT\cite{cub} & DDPM & {1.2} & - & 17.4 & 0.93 & \textbf{0.33} & 25.4 \\
        \rowcolor{Gray}
        \cellcolor{white}& CB-DDPM (ours) & {1.6} & 0.01  & \textbf{16.3} & \textbf{0.93} & 0.26 & \textbf{40.3} \\
        \bottomrule
        \end{tabular}}
        \caption{\small CBDM performance on high-resolution datasets. Here, CB-DDPM refers to the DDPM model fine-tuned by our method, $\omega$ refers to the guidance strength and $\tau$ refers to the weight of the regularization term. We note than the fine-tuning is applied on CUB and ImageNet-LT due to the limited calculation budget.}
      \label{tab:high_res}
\end{table}

We also investigated our methods' performance on commonly encountered datasets with higher resolution: CUB-200\cite{cub} (of resolution $128$), CelebA-5\cite{Kim_2020_CVPR_celeba5} (of resolution $64$) and ImageNet-LT\cite{openlongtailrecognition} (of resolution $64$).
Note that, the sampling size for evaluation is based on its corresponding training set (except Imagenet-LT, which uses 50k samples for evaluation) size for correctly calculating some metrics. From the results in \Tabref{tab:high_res}, we demonstrate CBDM and the fine-tuned model are consistently better than DDPM except on ImageNet-LT.
ImageNet-LT has lower Recall when using CBDM, which may be attributed to the limited size of its evaluation dataset (50k images). In contrast, its FID and IS metrics are derived from the much larger training set statistics, making them more reliable.
Also, similar to our observation before, the improvement on the imbalanced dataset (CelebA-5 \& Imagenet-LT) is more obvious than on the balanced dataset (CUB). 
We noticed that the regularization weight should be chosen with more caution when training models with larger resolution in order to avoid potential mode collapse. As marked in \Tabref{tab:high_res}, we use some small values such as $0.1$ and $0.01$ for experiments.

\paragraph{Trainable classifier-free guidance (TCFG)}

One drawback of classifier-free guidance (CFG) lies in its sampling speed. As CFG requires calling the model both in the unconditional and conditional case, CFG doubles the time complexity of common diffusion models during the sampling stage. Therefore, we tried to solve this issue by adding another conditional layer in the backbone to encode the guidance strength information. Precisely, we sample randomly the guidance strength $\omega$ and add another two loss terms to the model:

\begin{align}
&\gL_{g}(\rvx_{t}, y)
=||[\rvepsilon_{\theta}(\rvx_{t} ,y, \omega) - \omega(\rvepsilon_{\theta}(\rvx_{t} ,y)-\rvepsilon_{\theta}(\rvx_{t-1})).sg()]-\rvepsilon||^2\\
&
\gL_{gc}(\rvx_{t}, y)
=\frac{1}{4}||[\rvepsilon_{\theta}(\rvx_{t} ,y, \omega).sg() - \omega(\rvepsilon_{\theta}(\rvx_{t} ,y)-\rvepsilon_{\theta}(\rvx_{t-1}))]-\rvepsilon||^2
\end{align}

Those two losses decrease significantly the sampling time at the price of a longer training time. 
\Tabref{fig:TCFGcifar100} shows the comparison of using CFG and TCFG for sampling. We trained the DDPM as well as the TCFG model on the CIFAR100 dataset respectively with the same number of iterations for both.
Given the large number of parameters to be tested for visualization, we generated 10k images for each setting and tested 6 guidance strength within 0 to 1.

\begin{figure}[h]
    \centering
    \begin{subfigure}[b]{0.49\linewidth}
         \centering
         \includegraphics[width=0.75\linewidth]{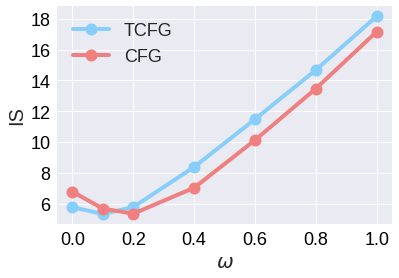}
         \caption{FID}
    \end{subfigure}
    \hfill
    \begin{subfigure}[b]{0.49\linewidth}
         \centering
         \includegraphics[width=0.75\linewidth]{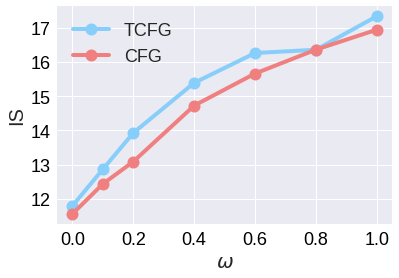}
         \caption{IS}
    \end{subfigure}
    \caption{CFG/TCFG comparison with an embedded $\omega$
    }
      \label{fig:TCFGcifar100}
\end{figure}

The \Figref{fig:TCFGcifar100} shows that the FID and IS scores of TCF and TCFG are quite comparable. Although TCFG shows a slightly higher FID compared to TCF, it demonstrates a slightly better IS value, which implies that TCFG prioritizes improving the image quality over diversity. However, if TCFG is shifted by 0.1 units, the curves display a high degree of similarity with CFG. This shift occurs because TCFG training utilizes all backbone parameters to train a model under difference guidance, which makes the model's parameters move closer to the guided generation results. Additionally, when the guidance strength of TCFG is set to 0.1, the FID score of TCFG ($\omega$=0.1, FID=5.343) is slightly better than that of the CFG approach at the optimal guidance strength ($\omega$=0.2, FID=5.357). In conclusion, the TCFG method demonstrates a high degree of similarity to the CFG method but maintained the sampling speed equivalent to the unguided generation. Moreover, the TCFG method is not difficult to implement and can be combined with any diffusion model using CFG guidance. However, the primary challenge would be the reduction in the training speed. For time limit, we only conduct TCFG experiments on CIFAR100 dataset, but we encourage testing this trick in practice when the sampling speed is an important issue.

\paragraph{Unconditional training probability}

Another hyperparameter in CBDM worth exploring is the probability of unconditional generation involved in CFG. According to the original paper, we performed unconditional generation with a probability of 10\% and conditional generation with a probability of 90\%. In our experiments, we increased the value of this probability and tried $20\%$ and $50\%$ as the new unconditional generation probabilities. As shown in the following table \Tabref{tab:phi}, we find that the performance of CBDM is optimal when $\phi$ equals to 0.1.

\begin{table*}[h]
  \centering
  \begin{tabular}{l|ll}
    \toprule
     $\phi$ & FID & IS  \\
    \midrule
    $10\%$  & 8.299&12.457\\
    $20\%$ & 9.161&12.261\\
    $50\%$  & 10.392&11.510\\
    \bottomrule
  \end{tabular}
    \caption{Influence of $\phi$ to model performance}
  \label{tab:phi}
\end{table*}

\section{Additional Discussions}

\paragraph{Relationship between Classifier-Free Guidance and Logit Adjustment}
It is not difficult to see that the structure of the adjustment form of CBDM and logit adjustment is similar. 
In classification tasks, logit adjustment can be implemented in the form of post-hoc adjustment or in the form of training. Precisely, by noting the classifier as $f$, the frequency of class $y$ as $\pi(y)$ and the adjustment weight as $\tau$, we have:

\begin{itemize}
    \item Post-hoc: $f^\star_{y}(x)=f_{y}(x) + \tau\pi (y)$
    \item Training: $f^\star_{y}(x)=f_{y}(x) - \tau\pi (y)$
\end{itemize}

Similarly, for the diffusion model, we reverse the adjustment term as in Prop. (\ref{theo:adj}) in order to realize a training adjustment, and we have:

\begin{itemize}
    \item Post-hoc: $\rvepsilon_\theta(\rvx_{t-1}, y) = \rvepsilon_\theta(\rvx_{t-1}, y) + \tau \rvepsilon_{adj}$
    \item Training: $\rvepsilon_\theta(\rvx_{t-1}, y) = \rvepsilon_\theta(\rvx_{t-1}, y) - \tau \rvepsilon_{adj}$
\end{itemize}

Although the idea of this adjustment is very simple, it is convenient to transfer some cumbersome post-hoc adjustment from sampling to training.